%% file: neurips_2026.tex
\documentclass{article}

\PassOptionsToPackage{numbers, compress}{natbib}

\usepackage[preprint]{neurips_2026}

\usepackage[utf8]{inputenc} 
\usepackage[T1]{fontenc}    
\usepackage{hyperref}       
\usepackage{url}            
\usepackage{booktabs}       
\usepackage{amsfonts}       
\usepackage{nicefrac}       
\usepackage{microtype}      
\usepackage{xcolor}         

\usepackage{inconsolata}
\usepackage{adjustbox}
\usepackage{amsmath}
\usepackage{bm}

\usepackage{enumitem}
\usepackage{amssymb}
\usepackage{amsthm}
\usepackage{cleveref}
\usepackage{wrapfig}
\usepackage{stackengine}
\usepackage{comment}
\usepackage{capt-of}
\usepackage{mathtools}
\usepackage{caption}
\usepackage{tikz}
\usepackage{xspace}
\usepackage{graphicx}

\newcommand{\appref}[1]{\textcolor{black}{\hyperref[#1]{Appendix~\ref*{#1}}}}

\theoremstyle{plain}

\newtheorem{lemma}{Lemma}
\newtheorem{corollary}{Corollary}

\newtheorem{proposition}{Proposition}

\theoremstyle{remark}

\definecolor{LightSteelBlue1}{RGB}{202,225,255}
\definecolor{LightPink}{RGB}{245,191,210}
\definecolor{Moccasin}{RGB}{255, 228, 181}

\definecolor{LightSteelBlue1}{RGB}{202,225,255}
\definecolor{LightPink}{RGB}{245,191,210}
\definecolor{Moccasin}{RGB}{255, 228, 181}
\setenumerate[1]{itemsep=0pt,partopsep=0pt,parsep=\parskip,topsep=0pt}
\setitemize[1]{itemsep=0pt,partopsep=0pt,parsep=\parskip,topsep=0pt}
\setdescription{itemsep=0pt,partopsep=0pt,parsep=\parskip,topsep=0pt}

\definecolor{cmzhao}{rgb}{0.1, 0.8, 0.1}

\newcommand{\mtd}{SMoA\xspace}

\title{SMoA: Spectrum Modulation Adapter for Parameter-Efficient Fine-Tuning}

\author{
  Yongkang Liu$^{1}$
  Xing Li$^{1}$
  Mengjie Zhao$^{1}$\thanks{Corresponding authors.}\hspace{0.3em}
  Shanru Zhang$^{1}$
  Zijing Wang$^{1}$ \\
  \textbf{Qian Li}$^{2}$
  \textbf{Shi Feng}$^{1}$
  \textbf{Feiliang Ren}$^{1}$
  \textbf{Daling Wang}$^{1}$
  \textbf{and Hinrich Schütze$^{3,4}$} \\
  $^{1}$Northeastern University, China; $^{2}$Shandong University, China;\\
  $^{3}$CIS, LMU Munich, Germany;$^{4}$MCML, Germany
}

\begin{document}

\maketitle

\begin{abstract}
As the number of model parameters increases, parameter-efficient fine-tuning (PEFT)
has become the go-to choice for tailoring pre-trained large language models. Low-rank Adaptation 
(LoRA) uses a low-rank update method to simulate full parameter fine-tuning, which is widely 
used to reduce resource requirements. However, decreasing the rank encounters challenges with 
limited representational capacity. Theory suggests that LoRA fine-tuning with rank $r$ converges toward the top $r$ singular values of the pre-trained weight matrix. As the rank increases, more principal singular directions are preserved, which generally improves the model’s performance. However, a larger rank also introduces more trainable parameters, leading to higher computational cost. To overcome this dilemma, we propose SMoA, a \textbf{S}pectrum \textbf{Mo}dulation \textbf{A}dapter that enlarges the accessible family of spectrum-aware updates under a smaller parameter budget. SMoA partitions the layer into multiple aligned spectral blocks and applies one in-block Hadamard-modulated low-rank branch to each diagonal block, yielding broader coverage of pretrained spectral directions. We provide theoretical analysis and empirical results on multiple tasks. In our experiments, SMoA improves average performance in the current lower-budget setting over LoRA and competitive LoRA-style baselines.
\end{abstract}

\section{Introduction}

Large language models (LLMs)~\cite{bai2023qwen,liu2024deepseek,zhang2022opt,achiam2023gpt} have demonstrated remarkable performance improvements across a wide range of natural language 
processing tasks. Fine-tuning (FT) is a standard approach for adapting LLMs to specific downstream tasks~\cite{ouyang2022training,weifinetuned,liu2025look}. However, the sheer scale of modern LLMs makes full-parameter FT prohibitively expensive, especially in resource-constrained environments~\cite{hoffmann2022training,kaplan2020scaling}. Parameter-Efficient Fine-Tuning (PEFT) is introduced to reduce this cost~\cite{houlsby2019parameter,li2021prefix,lester2021power,liu2024gpt,hu2022lora,hayou2024lora+}. PEFT lowers the memory and optimization burden by restricting training to a small set of additional parameters.

Low-rank adaptation~\cite{hu2022lora} is widely used because it introduces very little additional memory overhead and no extra inference latency. However, previous studies~\cite{jiang2024mora, liu2024dora, zhuang2024time} have shown that LoRA and most of its variants~\cite{lialin2023relora, hayou2024lora+} can struggle on complex tasks such as mathematical reasoning and learning new knowledge~\cite{liu2025look}. Since a rank-$r$ LoRA update is restricted to rank at most $r$, adaptation is confined to a limited set of global spectral directions~\cite{rushka2026rank}.
Gradient-flow theory shows that, for low-rank approximation objectives, rank-$r$ LoRA converges to the optimal rank-$r$ truncated SVD solution, so the residual error is governed by the singular-value tail beyond rank $r$~\cite{rushka2026rank}. Meanwhile, random-matrix analyses suggest that transformer weights often have long-tailed spectra and numerical rank close to $\min(d_{\text{out}},d_{\text{in}})$, with even small singular directions remaining functionally relevant~\cite{staats2024small}. These results suggest a structural limitation of a single global LoRA branch: when useful signal persists in the spectral tail, restricting the update to rank $r$ leaves potentially informative directions under-modeled.

A natural response is to enlarge the accessible update family beyond a single global low-rank branch. Existing methods such as ReLoRA~\cite{lialin2023stack}, COLA~\cite{xia2024chain}, MELoRA~\cite{ren2024melora}, and HiRA~\cite{huang2025hira} move in this direction by stacking, composing, or modulating LoRA-style updates. However, they do not explicitly organize adaptation around the spectral structure of the pretrained operator, which may limit their ability to capture informative tail directions under a small parameter budget.

In this paper, we propose a \textbf{S}pectrum \textbf{Mo}dulation \textbf{A}dapter (SMoA) that explicitly targets this spectral regime. SMoA partitions both the output and input dimensions into $K$ aligned spectral blocks, applies one Hadamard-modulated low-rank branch to each diagonal block, and combines the resulting local updates into one global block-diagonal residual. This construction distributes the PEFT budget across 
$K$
spectrum-aware anchors, allowing the update to interact with a broader collection of singular directions than a single global rank-$r$ factorization. Under near-full-rank block structure and $K \ll \min(d_{\text{out}},d_{\text{in}})$
, our analysis shows that SMoA admits a strictly larger rank ceiling than rank-$r$ LoRA. We evaluate SMoA on Llama-2-7B and Llama-3-8B across commonsense reasoning, dialogue generation, and mathematical reasoning tasks. The results suggest stronger average performance in the lower-budget setting considered here.

We summarize our contributions as follows: \textbf{(1)} We propose SMoA, which partitions both dimensions of a pretrained operator into aligned spectral blocks and applies one Hadamard-modulated low-rank branch of local rank $r/K$ to each diagonal block.
\textbf{(2)} We theoretically characterize the parameter complexity and rank ceiling of SMoA, show that under near-full-rank block structure its rank ceiling is strictly larger than that of rank-$r$ LoRA, identify a family of block-aligned modulated targets that are representable by SMoA but not by rank-$r$ LoRA, and show that on this witness family rank-$r$ LoRA suffers an irreducible spectral approximation gap while SMoA attains zero error.
\textbf{(3)} Experimental results suggest that SMoA attains stronger average performance than LoRA and remains competitive with strong LoRA variants across multiple benchmarks in the lower-budget setting studied here.


\begin{figure*}[t!]
    \centering   \includegraphics[width=.9\textwidth]{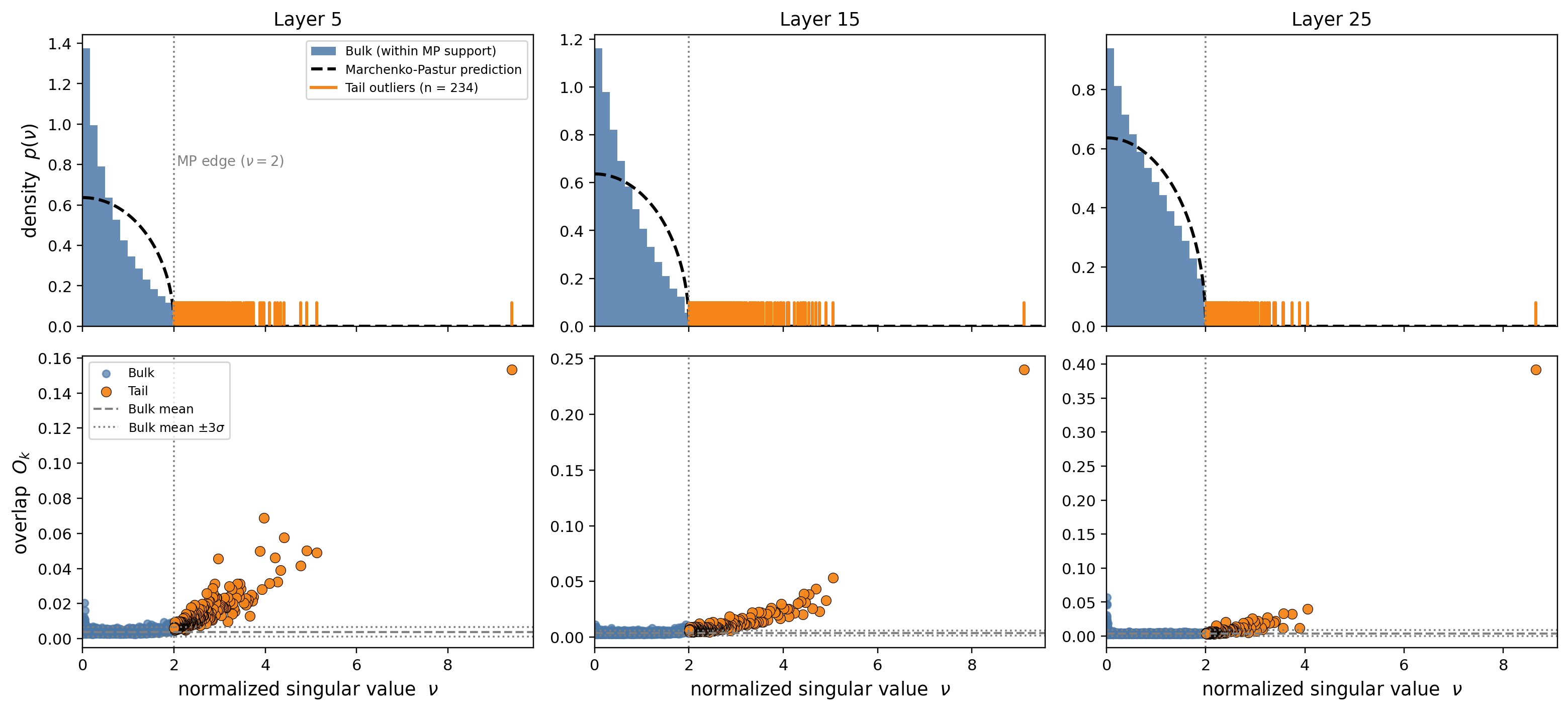}
    \caption{
    \textbf{Pretrained weights have informative spectral tails.}
    Top: Distribution of normalized singular values $\nu$ for Llama-2-7B's layer 5, 15, and 25. The black dashed curve shows Marchenko–Pastur prediction for a random matrix of the same shape; values beyond the bulk edge
    (vertical line) 
    are tail outliers encoding structured and task-relevant information.
    Bottom: Maximum overlap $\mathcal{O}_k$ between right singular vectors and activation-covariance eigenvectors. Tail singular directions (orange) still show substantially larger $\mathcal{O}_k$, indicating alignment with task-induced activation structure. Tail directions are therefore not noise but carry functionally relevant signal.
    }
    \label{fig:SpectraMatrices}
    \vspace{-.5cm}
\end{figure*}

\section{Related Work}

The scale of modern large language models (LLMs) has made parameter-efficient fine-tuning (PEFT) a central research topic~\cite{dettmers2023qlora,han2024parameter}. PEFT aims to approach the performance of full fine-tuning while updating only a small subset of model parameters~\cite{liu2025look}. Existing methods are commonly grouped into prompt-based~\cite{li2021prefix,liu2022p}, adapter-based~\cite{houlsby2019parameter}, selection-based~\cite{zaken2022bitfit}, and low-rank adaptation approaches~\cite{zhang2023adalora}, among which low-rank adaptation has become one of the dominant paradigms.

LoRA~\cite{hu2022lora} represents weight updates as the product of two low-rank matrices, achieving strong efficiency without increasing inference cost. Building on this framework, VeRA~\cite{kopiczko2023vera} learns only scaling vectors on top of frozen random bases, DoRA~\cite{liu2024dora} decouples direction and magnitude in the update, FLoRA~\cite{wenbatched} introduces Hadamard-product-based example-specific adaptation, and SSMLoRA~\cite{yu2025ssmlora} connects low-rank adapters across layers via a state-space mechanism.

Another line of work aims to improve the expressive capacity of low-rank adaptation. MoRA~\cite{jiang2024mora} replaces the standard low-rank update with a square higher-rank matrix. MELoRA~\cite{ren2024melora} ensembles multiple small low-rank adapters, while HiRA~\cite{huang2025hira} employs element-wise products to realize effectively higher-rank task-specific updates. CoTo~\cite{zhuang2025come} improves adaptation by progressively expanding adapter activation during training, encouraging broader exploration of the optimization landscape.

Our work is closest in spirit to methods that seek to overcome the limitations of a single global low-rank branch, but differs in both motivation and construction. Rather than increasing rank through more complex parameterizations or combining multiple LoRA modules, SMoA is motivated by the spectral structure of pretrained Transformer weights, which are often long-tailed and numerically near full-rank. Accordingly, SMoA performs spectrum-aware reordering and allocates adaptation across multiple aligned local subspaces. Each local branch is anchored by a pretrained spectral block and modulated via a Hadamard product, enabling rank accumulation across blocks under a structured spectrum-aware inductive bias. This makes SMoA a spectrum-aware alternative to standard LoRA-style global low-rank adaptation.

\section{Methodology}
\label{sec:method}
We first explain why pretrained transformer weights are better viewed as numerically near full-rank matrices with informative spectral tails than as effectively low-rank operators. This motivates an adapter that can preserve or modulate more than only a few leading singular directions. We then construct SMoA by distributing a small PEFT budget across aligned spectral blocks of the frozen weight, and finally analyze its parameter count, potential rank ceiling, and expressivity relative to a rank-$r$ LoRA layer.

\subsection{Spectral Modulation}

\begin{wrapfigure}{r}{0.6\textwidth}
    \vspace{-1.5em}
  \centering
  \includegraphics[width=0.7\textwidth]{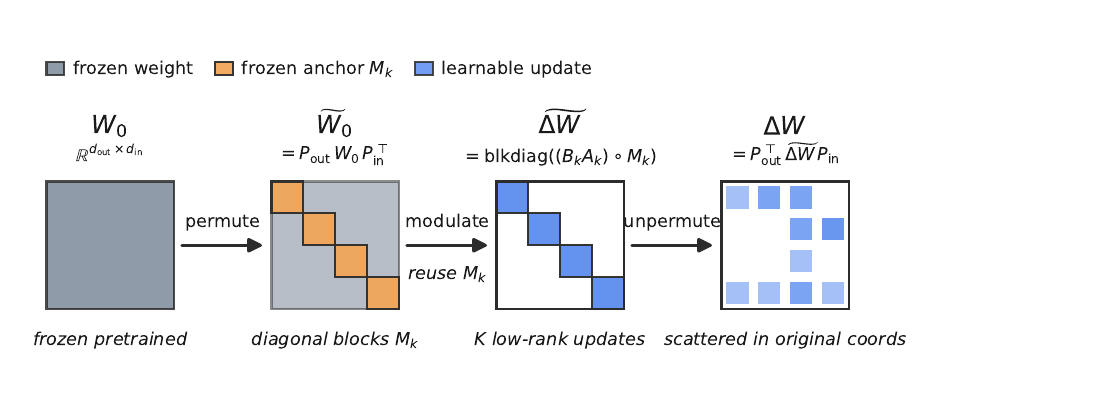}
  \captionsetup{skip=-1em} 
  \caption{Overview of \mtd pipeline (when $K=4$).}
  \label{fig:smoa_pipeline}
  \vspace{-1em}
\end{wrapfigure}

We illustrate the informative spectral tails of pretrained model weights in Figure~\ref{fig:SpectraMatrices},
following the random matrix-based analysis of Transformer models \cite{staats2024small}.
Figure~\ref{fig:SpectraMatrices} top images highlight that a pretrained weight's singular-value spectrum has a heavy tail: 
a substantial fraction of singular values lies beyond the Marchenko-Pastur bulk \cite{marvcenko1967distribution} predicted for a same-size random matrix, confirming that the pretrained weight indeed encodes task-relevant information in the spectrum tail. Bottom images further confirm that directions on the tail are not random noise and they have right singular vectors aligning more (exceeding $3\sigma$ intervals) with the activation-covariance eigenvectors (see below) than the bulk directions. Overall, these observations indicate that pretrained weights are numerically near full-rank with informative tails, consistent with prior random-matrix analyses \cite{staats2024small}.

We record above empirical regularity as Lemma~\ref{lem:empirical-near-full-rank} (\appref{app:proof-lemma1}).
If informative directions extend into the spectral tail, restricting the update to a rank-$r$ subspace leaves potentially useful directions under-modeled.
Motivated by this, SMoA is designed to preserve and exploit a broader range of spectrum-aware structure than a single global low-rank branch. For a frozen linear layer $W_0 \in \mathbb{R}^{d_{\text{out}} \times d_{\text{in}}}$, we construct SMoA in three steps shown in Figure~\ref{fig:smoa_pipeline}: (i) reorder the coordinates once according to the frozen spectral structure of $W_0$, (ii) build one Hadamard-modulated low-rank branch on each aligned diagonal block, and (iii) map the resulting block-diagonal update back to the original coordinates. Additional preprocessing details appear in~\appref{app:method-preprocess}, while the deferred theoretical statements and proofs appear in~\appref{app:proof-full-rank-cor},~\appref{app:proof-expressivity-separation}, and~\appref{app:proof-spectral-gap}.

\paragraph{Spectrum-aware preprocessing.}
Let $W_0 = U \Sigma V^\top$ be the singular value decomposition of the pretrained weight. We use this frozen spectral structure only once to define output and input permutations, denoted by $\pi_{\text{out}}$ and $\pi_{\text{in}}$, that reorder the coordinates into $K$ spectrum-aligned contiguous groups. 
The corresponding permutation matrices are
\small
\begin{equation}
(P_{\mathrm{out}})_{ij} = \mathbf{1}[j = \pi_{\mathrm{out}}(i)]
    \;\; \text{for } i,j \in \{1, \dots, d_{\mathrm{out}}\},
\quad
(P_{\mathrm{in}})_{ij} = \mathbf{1}[j = \pi_{\mathrm{in}}(i)]
  \;\; \text{for } i,j \in \{1, \dots, d_{\mathrm{in}}\}.
\label{eq:perm-matrices}
\end{equation}
\normalsize


The reordered weight matrix
$\widetilde{W}_0 \in \mathbb{R}^{d_{\mathrm{out}} \times d_{\mathrm{in}}}$ is then
\small
\begin{equation}
  \widetilde{W}_0 = P_{\mathrm{out}} W_0 P_{\mathrm{in}}^\top,
  \qquad
  (\widetilde{W}_0)_{ij} = (W_0)_{\pi_{\mathrm{out}}(i),\, \pi_{\mathrm{in}}(j)}
    \;\; \text{for } i \in \{1, \dots, d_{\mathrm{out}}\},\; j \in \{1, \dots, d_{\mathrm{in}}\}.
  \label{eq:reordered-weight}
\end{equation}
\normalsize

This reordering leaves the layer computation unchanged; it only exposes a coordinate system in which related spectral coordinates become contiguous.

\paragraph{Block anchors.}
SMoA places its local branches on 
aligned
diagonal blocks of $\widetilde W_0$. For exposition, assume that $K$ divides both $d_{\text{out}}$ and $d_{\text{in}}$, and let $\mathcal{O}_k$ and $\mathcal{I}_k$ denote the $k$-th contiguous row and column intervals after reordering, each of size $d_{\text{out}}/K$ and $d_{\text{in}}/K$, respectively. The $k$-th diagonal anchor is then
$
 M_k = \widetilde W_0[\mathcal{O}_k,\mathcal{I}_k]
 \in \mathbb{R}^{\frac{d_{\text{out}}}{K}\times \frac{d_{\text{in}}}{K}}.
$ (see orange blocks in Figure~\ref{fig:smoa_pipeline}).
We refer to $\{M_k\}_{k=1}^{K}$ as the reordered block anchors of $W_0$. These anchors are fixed after preprocessing and are not updated during fine-tuning. Note that $\widetilde W_0$ itself need not be block diagonal; SMoA only uses its aligned diagonal blocks as local spectral anchors. Additional indexing details are deferred to~\appref{app:method-preprocess}.

\paragraph{Adapter update.}
Let the reference LoRA baseline use a global rank budget $r$. To preserve the same total rank budget while distributing adaptation capacity across $K$ spectrum-aligned blocks, we define the \textbf{local block rank} as $\rho = \frac{r}{K}$. SMoA then attaches one local branch of rank $\rho$ to each diagonal anchor block.
Concretely, for each
block $M_k$, 
we introduce learnable low-rank factors
$
    A_k \in \mathbb{R}^{\rho\times \frac{d_{\text{in}}}{K}},
    B_k \in \mathbb{R}^{\frac{d_{\text{out}}}{K}\times \rho},
$
and define the corresponding Hadamard-modulated local update by
$
    \Delta M_k = (B_kA_k)\odot M_k.
$
The full update in reordered coordinates is
\begin{equation}
\widetilde{\Delta W}
=
\operatorname{blkdiag}(\Delta M_1,\dots,\Delta M_K)
=
\operatorname{blkdiag}\big((B_1A_1)\odot M_1,\dots,(B_KA_K)\odot M_K\big).
\end{equation}
Scattering
this update back to the original coordinates yields
\small
\begin{equation}
\Delta W = P_{\text{out}}^\top\, \widetilde{\Delta W}\, P_{\text{in}},
\qquad
W = W_0 + \Delta W.
\end{equation}
\normalsize

Thus, $(P_{\text{out}},P_{\text{in}},\{M_k\}_{k=1}^{K})$ remain fixed, while the trainable parameters are $\{A_k,B_k\}_{k=1}^{K}$. Compared with LoRA, SMoA distributes the adaptation budget across $K$ local spectrum-aligned branches instead of a single global low-rank branch. 

\subsection{Analysis: Spectrum-Aware Rank Capacity and Expressivity}
\label{sec:32}

We next compare SMoA with a standard rank-$r$ LoRA layer in parameter count and achievable rank capacity; fuller witness-family separation results are deferred to~\appref{app:proof-expressivity-separation} and~\appref{app:proof-spectral-gap}.

\paragraph{Parameter complexity.}

Compared with the $r(d_{\text{in}}+d_{\text{out}})$ trainable parameters of a rank-$r$ LoRA layer, SMoA with $K$ blocks and local rank $\rho=\frac{r}{K}$ uses
\small
\begin{equation}
\label{eq:param-match}
P_{\mathrm{SMoA}}
=
\sum_{k=1}^{K}
\left(
\rho\frac{d_{\text{in}}}{K}
+
\rho\frac{d_{\text{out}}}{K}
\right)
=
\frac{r}{K}(d_{\text{in}}+d_{\text{out}})
=
\frac{1}{K}r(d_{\text{in}}+d_{\text{out}}),
\end{equation}
\normalsize
that is, only a $\frac{1}{K}$ fraction of the trainable parameters required by rank-$r$ LoRA.
Next, we show the main rank-capacity consequence of \mtd's block-wise construction.

\begin{proposition}[{Spectrum-Aware Rank Ceiling}]
\label{prop:smoa-rank-ceiling}

Let
\begin{equation}
\widetilde{\Delta W}
=
\operatorname{blkdiag}(\Delta M_1,\dots,\Delta M_K),
\qquad
\Delta M_k=(B_kA_k)\odot M_k,
\qquad
\rho=\frac{r}{K}.
\end{equation}
Here $A_k\in\mathbb{R}^{\rho\times \frac{d_{\text{in}}}{K}}$, $B_k\in\mathbb{R}^{\frac{d_{\text{out}}}{K}\times \rho}$, $M_k\in\mathbb{R}^{\frac{d_{\text{out}}}{K}\times \frac{d_{\text{in}}}{K}}$, and $s_k=\min\!\left(\frac{d_{\text{out}}}{K},\frac{d_{\text{in}}}{K}\right)$.
Then, for each block,
\begin{equation}
\operatorname{rank}(\Delta M_k)
\le
\min\!\left(
s_k,\;
\rho\,\operatorname{rank}(M_k)
\right).
\end{equation}
Consequently, the full SMoA update satisfies
\begin{equation}
\label{eq:merged-rank-ceiling}
\operatorname{rank}(\Delta W)
=
\operatorname{rank}(\widetilde{\Delta W})
\le
U
:=
\sum_{k=1}^{K}
\min\!\left(
s_k,\;
\rho\,\operatorname{rank}(M_k)
\right)
\le
\min(d_{\text{out}},d_{\text{in}}).
\end{equation}
In contrast, a standard rank-$r$ LoRA update always satisfies $\operatorname{rank}(\Delta W_{\mathrm{LoRA}})\le r$. Therefore
\begin{equation}
\label{eq:merged-separation-criterion}
U>r,
\end{equation}
SMoA admits a strictly larger analytic rank ceiling than rank-$r$ LoRA.

\end{proposition}
Please refer to~\appref{app:proof-rank-ceiling} for the proof.
\appref{app:proof-full-rank-cor} Corollary~\ref{cor:near-full-rank-ceiling} also shows 
a useful special case:
under full-rank local anchors, this ceiling is strictly larger than that of a standard rank-$r$ LoRA update. Proposition~\ref{prop:expressivity-separation} constructs a block-aligned witness family that is exactly representable by SMoA but can fall outside the rank-$r$ LoRA family when its reordered rank exceeds $r$. Corollary~\ref{cor:spectral-gap} in~\appref{app:proof-spectral-gap} then shows that, on this same witness family, the best rank-$r$ LoRA approximation incurs the truncated-spectrum residual $\sum_{j>r}\sigma_j(\widetilde{\Delta W}^{\star})^2$, whereas SMoA attains zero error. This is consistent with the gradient-flow theory of \citet{rushka2026rank}: under Frobenius low-rank approximation objectives and spectral initialization, LoRA converges to the optimal rank-$r$ truncated-SVD solution, so the residual is a structural consequence of the restricted update family rather than merely an optimization artifact.

Taken together, Equation~\eqref{eq:merged-rank-ceiling}, Proposition~\ref{prop:smoa-rank-ceiling}, and the appendix results characterize the expressive advantage of SMoA from three complementary perspectives: a smaller parameter budget via Equation~\eqref{eq:param-match}, a larger analytic rank ceiling via Proposition~\ref{prop:smoa-rank-ceiling}, and a one-sided witness-family separation via Proposition~\ref{prop:expressivity-separation}. 
\appref{app:proof-expressivity-separation} and~\appref{app:proof-spectral-gap}
show the detailed witness construction and proofs.
Because SMoA remains block diagonal in the reordered coordinates, this advantage is structural rather than universal; if a task residual is dominated by strong cross-block interactions, the same inductive bias can become restrictive.

\begin{table*}[ht]
\captionsetup{skip=2pt}
\caption{Accuracy comparison among PEFT methods. Results for 
ChatGPT are sourced from~\cite{liu2024dora}. The best performance within is indicated in \textbf{bold},
while the second-best performance is highlighted with underlining. The reported numbers for SMoA are averaged accuracy (standard deviation).}
\label{tab:commonsense}
\begin{adjustbox}{max width=0.9\textwidth, center}
\includegraphics[width=\textwidth]{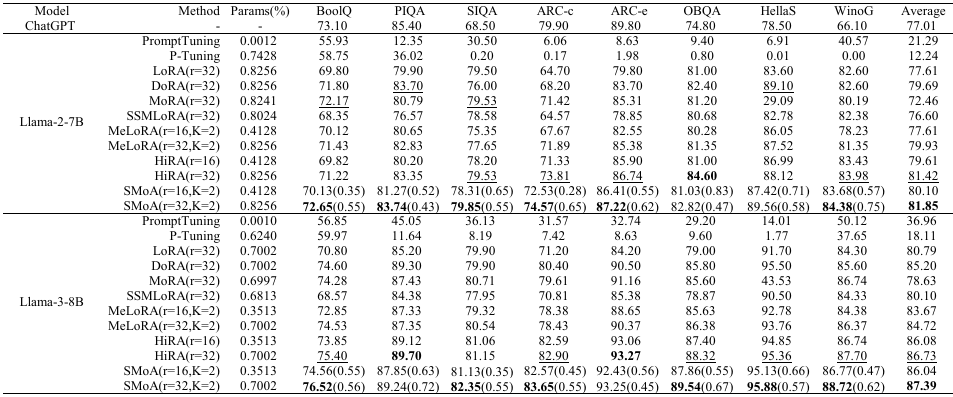}
\end{adjustbox}
\vspace{-2.0em}
\end{table*}

\section{Experiments}
\label{sec:experiments}

We conduct experiments on three task families: commonsense reasoning, dialogue generation, and mathematical reasoning. 
For commonsense reasoning, we evaluate on eight standard benchmarks: BoolQ~\cite{clark2019boolq}, PIQA~\cite{Bisk2020}, SIQA~\cite{sap2019socialiqa}, ARC-c~\cite{clark2018think}, ARC-e~\cite{clark2018think}, OBQA~\cite{mihaylov2018can}, HellaSwag~\cite{zellers2019hellaswag}, and WinoGrande~\cite{ai2:winogrande}. 
For dialogue generation and mathematical reasoning, we use ConvAI2~\cite{dinan2019second} and GSM8K~\cite{cobbe2021gsm8k}, respectively. 
We compare \textsc{SMoA} with representative PEFT methods, including prompt-based approaches and recent LoRA variants, and report the standard metrics for each task family. 
To keep the main text focused, we provide the complete experimental details in the appendix:~\appref{app:data} describes the datasets and task formulations;~\appref{app:baselines} introduces the compared baselines;~\appref{app:evaluation_metric} specifies the evaluation metrics and answer-extraction protocols; and~\appref{app:implementation_details} reports the training setup, implementation choices, and hyperparameter configurations.

\begin{table*}[ht]
\captionsetup{skip=-1pt}
\caption{Performance comparison among various PEFT methods on the CONVAI2 dataset using BLEU, BERTScore F1/Recall/Precision, METEOR, ROUGE-L, and average score.}
\label{tab:results2}
\begin{adjustbox}{max width=\textwidth, center}
\includegraphics[width=\textwidth]{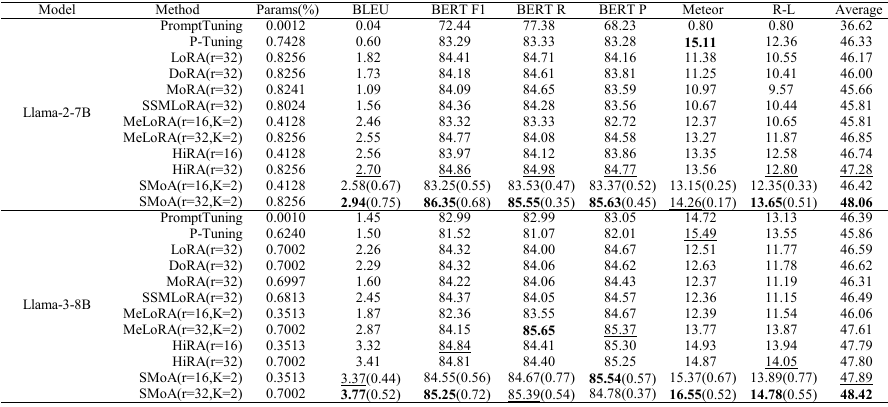}
\end{adjustbox}
\vspace{-2.5em}
\end{table*}

\subsection{Commonsense Reasoning Tasks}
Table~\ref{tab:commonsense} reports the accuracy comparison of various PEFT methods on a 
diverse set of commonsense reasoning benchmarks, including BoolQ, PIQA, SIQA, ARC-c, ARC-e, OBQA, 
HellaSwag, and WinoGrande, evaluated on the Llama-2-7B~\cite{touvron2023llama} and Llama-3-8B~\cite{grattafiori2024llama} backbones.
Across both backbones, SMoA($r=32$, $K=2$) gives the highest average accuracy in Table~\ref{tab:commonsense}. On Llama-2-7B, SMoA reaches 81.85\% on average, compared with 81.42\% for HiRA($r=32$) and 79.93\% for MeLoRA($r=32$, $K=2$). On Llama-3-8B, SMoA reaches 87.39\% on average, compared with 86.73\% for HiRA($r=32$) and 85.20\% for DoRA($r=32$). These gains should be read in the current lower-budget setting, since SMoA uses fewer trainable parameters than rank-$r$ LoRA. These results suggest that SMoA is most useful on tasks that benefit from broader \emph{structured} adaptation.

\begin{wraptable}{r}{0.40\columnwidth}
  \vspace{-0.5em}
  \centering
  \captionsetup{skip=2pt}
  \caption{Comparison of PEFT methods.}
    \label{tab:gsm8k}
    \includegraphics[width=\linewidth]{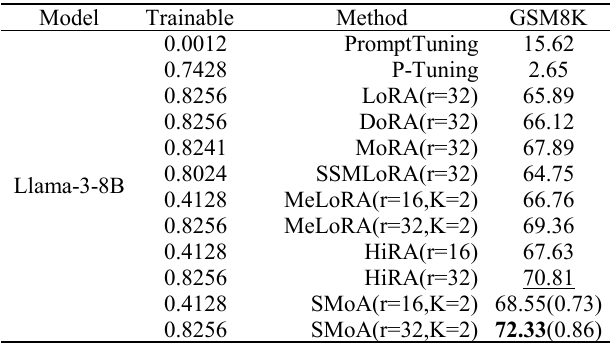}
\vspace{-2.0em}
\end{wraptable}

SMoA also remains competitive under more parameter-efficient settings. With $r=16$ and $K=2$, it attains an average accuracy of $80.10\%$ on Llama-2-7B while using substantially fewer trainable parameters than most rank-32 methods, yet still matching or exceeding several baselines. On Llama-3-8B, SMoA with $r=16$ reaches 86.04\% on average. These results suggest a favorable accuracy-efficiency trade-off.
Overall, the commonsense reasoning results indicate that SMoA is competitive across model scales and maintains a favorable balance of performance and parameter efficiency. 
We show that SMoA gives the best average performance in Table~\ref{tab:commonsense} while remaining competitive on most individual tasks.

\subsection{Conversational Task}
Table~\ref{tab:results2} reports a comprehensive comparison of various PEFT methods on 
the CONVAI2 dataset under two backbone models, Llama-2-7B and Llama-3-8B. 
Across both backbones, SMoA ($r=32$, $K=2$) yields the highest Average score, reaching 48.06\% on 
Llama-2-7B and 48.42\% on Llama-3-8B. This represents a clear improvement over strong baselines 
such as LoRA, HiRA, and MeLoRA with comparable parameter budgets. 
Overall, the results demonstrate that SMoA consistently achieves the best or near-best 
performance across most automatic evaluation metrics, validating its effectiveness for dialogue 
generation. Compared to LoRA and its variants (DoRA, MoRA, SSMLoRA), SMoA consistently delivers higher 
BLEU, METEOR, and ROUGE-L scores. This suggests that SMoA is more effective at capturing both 
surface-level n-gram overlap and longer-range semantic coherence.
Relative to HiRA and MeLoRA, which introduce hierarchical or multi-expert structures, SMoA 
demonstrates superior performance across metrics.
For example, although HiRA achieves strong BERT F1 scores, SMoA surpasses it in BLEU and METEOR, reflecting better lexical diversity and relevance 
in generated responses. 
This suggests that SMoA's multidimensional space adapter approach provides a more balanced 
trade-off between semantic fidelity and generation quality.

\subsection{Mathematical Reasoning Task}
We evaluate SMoA on GSM8K using Llama-3-8B as the backbone model. As shown in Table~\ref{tab:gsm8k}, SMoA reaches 72.33\% accuracy, compared with 65.89\% for LoRA, 66.12\% for DoRA, and 67.89\% for MoRA. Even under a smaller parameter budget ($r=16$, $K=2$), SMoA remains competitive at 68.55\%. These results suggest that SMoA adapts well to challenging tasks under limited resources.

Compared with LoRA and related methods, SMoA offers a useful balance between parameter efficiency and broader spectrum-aware adaptation, which may be particularly relevant for mathematical reasoning. Its performance is consistent with the view that broader spectrum-aware adaptation can enhance the model's capacity for mathematical reasoning. This observation should again be interpreted as task-dependent: the benefit comes from broader spectrum-aware adaptation, not from assuming that higher rank alone is always sufficient.

\begin{figure*}[t]
    \centering
    \includegraphics[width=\textwidth]{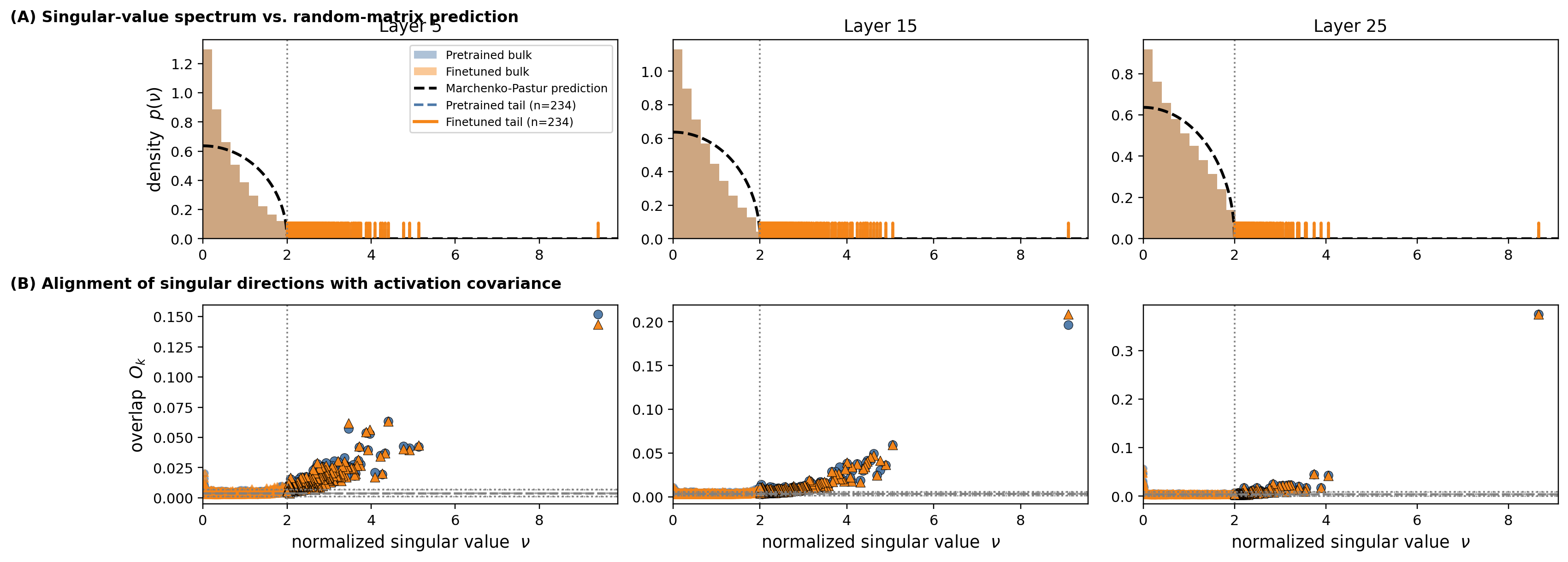}
    \captionsetup{skip=-1pt} 
    \caption{\textbf{Original Llama2-7B vs.\ SMoA finetuned spectra and overlaps on BoolQ for key weight matrix of attention.} We compare all layers before and after finetuning. The figure shows the 5th, 15th, and 25th layers, results for other layers are in the~\appref{sec:more_results}. Top: empirical distributions of normalized singular values $\nu$; light brown denotes the pretrained model and orange denotes the finetuned model. The dashed curve is the theoretical random-matrix bulk. Bottom: the corresponding maximum overlap $O_k$ between right singular vectors and activation-covariance eigenvectors, with circles for the pretrained model and triangles for the finetuned model. Gray lines indicate the bulk mean and $\pm 3\sigma$ intervals. The figure highlights how finetuning changes the right-tail spectrum and the alignment of tail singular directions with task-induced activation structure.}
    \label{fig:SpectraCompare}
    \vspace{-2em}
\end{figure*}

\subsection{Spectral Structure Comparison}

Figure~\ref{fig:SpectraCompare} reports three representative layers, while the complete layer-wise visualization is provided in \appref{sub:tail_spectrum_visualization}. 
A clear and consistent pattern emerges: compared with the pretrained model, whose spectrum is mostly confined to the random-matrix bulk, SMoA fine-tuning produces a pronounced spectral component beyond the bulk edge. 
This right-tail shift suggests that SMoA does not simply perturb dominant pretrained directions, but selectively activates underutilized singular directions that are likely to encode task-specific information.

The corresponding overlap scores provide functional evidence for this interpretation. 
Before fine-tuning, the singular directions exhibit weak alignment with the activation-covariance eigenspace, with most \(O_k\) values close to the bulk baseline. 
After SMoA fine-tuning, large overlaps appear mainly in the intermediate and tail spectral regions, indicating that the directions emphasized by SMoA are also more aligned with task-induced activation structure. 
Therefore, the spectral shift is not a random by-product of optimization, but reflects a structured adaptation mechanism.

These results highlight the advantage of SMoA: it preserves the pretrained spectral bulk while reallocating adaptation capacity to a small number of informative tail directions. 
This provides an explanation for why SMoA can achieve effective fine-tuning with limited trainable parameters. 
The same trend is observed across all layers in \appref{sub:tail_spectrum_visualization}, demonstrating that the phenomenon is systematic rather than specific to the selected layers.

\begin{table*}[t]
\captionsetup{skip=2pt}
\caption{Tail spectrum effectiveness ablation experiment. Ablation is performed only on the tail spectrum of the incremental weights $\Delta W$.}
\label{tab:tail_ablation}
\begin{adjustbox}{max width=0.9\textwidth, center}
\includegraphics[width=\textwidth]{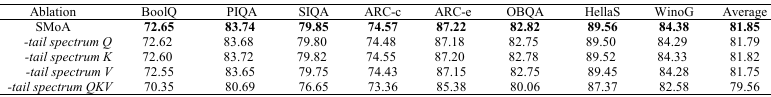}
\end{adjustbox}
\vspace{-2.0em}
\end{table*}
\begin{figure*}[t]
    \centering
    \includegraphics[width=0.85\textwidth]{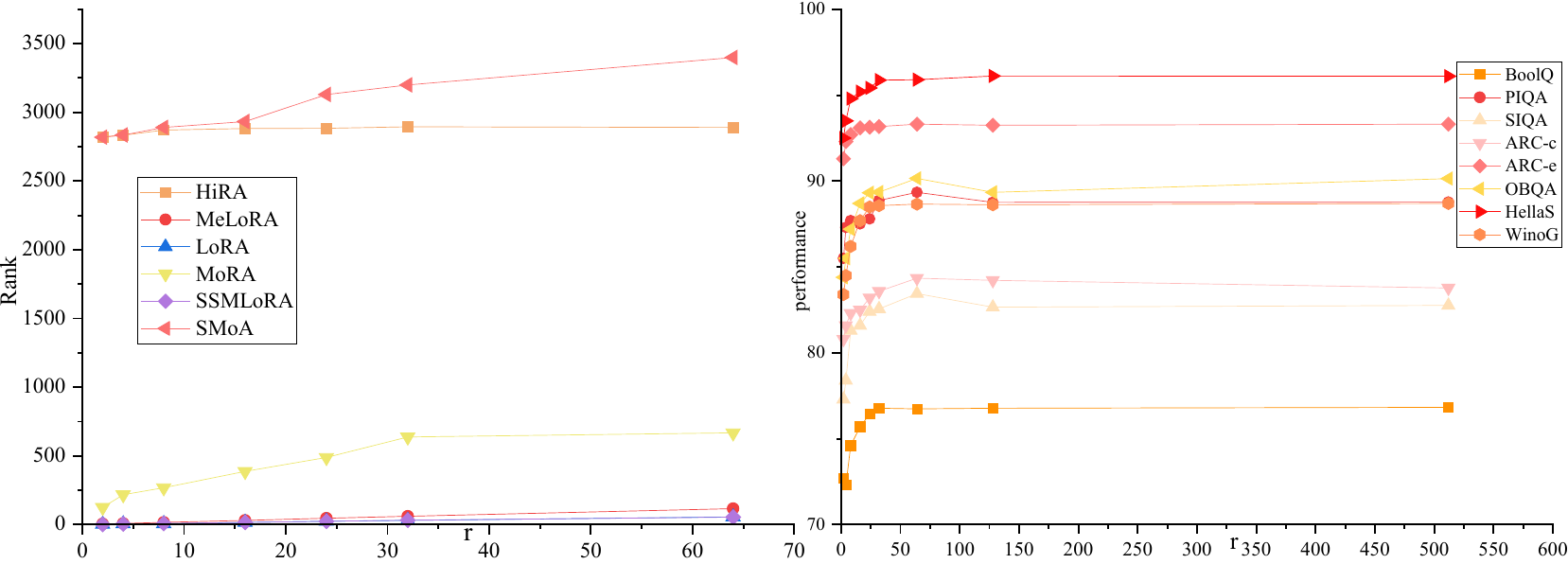}
    \captionsetup{skip=2pt} 
    \caption{Left: comparison of the measured update rank of $\Delta W$ for different PEFT methods after training. For MeLoRA and SMoA, the number of branches is set to 2. SMoA should be interpreted here as the lower-budget Hadamard-modulated block adapter defined in Section~\ref{sec:32}. Right: performance of SMoA across tasks as $r$ increases with $K=2$.}
    \label{fig:rank_pair}
    \vspace{-2em}
\end{figure*}
\subsection{Effectiveness of Tail Spectrum}

To further verify whether the tail singular spectrum contributes to the capability of SMoA, we conduct an ablation study by removing the tail-spectrum components from different attention projections. As shown in Table~\ref{tab:tail_ablation}, removing the tail spectrum from a single projection matrix, i.e., $Q$, $K$, or $V$, consistently leads to performance degradation across most benchmarks. The trend is highly consistent: the full SMoA model achieves the best or near-best performance on all tasks, while $-\textit{tail spectrum }Q$, $-\textit{tail spectrum }K$, and $-\textit{tail spectrum }V$ each slightly weaken the model.

More importantly, jointly removing the tail-spectrum components from $Q$, $K$, and $V$ causes a much larger degradation. For example, the average performance drops from $81.85\%$ for SMoA to $79.56\%$ after removing the tail spectrum from all three projections. This suggests that the tail spectrum is not merely redundant noise, but instead contains useful task-adaptive information that is distributed across multiple attention projections.
These results provide direct empirical evidence for our motivation: effective fine-tuning should not only exploit the dominant singular directions, but also preserve and adapt informative tail-spectrum components. The much stronger degradation under the joint $QKV$ ablation further indicates that the tail spectra of different projections play complementary roles in attention adaptation. Therefore, the superior performance of SMoA can be attributed, at least in part, to its ability to capture and utilize these otherwise overlooked tail-spectrum directions.

\begin{table*}[t]
\captionsetup{skip=2pt}
\caption{Performance on commonsense reasoning datasets under different SMoA block configurations $(r,K)$, using the same metrics as in Table~\ref{tab:commonsense}. Here $r$ denotes the reference LoRA rank and each SMoA block uses local rank $r/K$. Boldface indicates the best result for each metric.}
\label{tab:rank_analysis}
\begin{adjustbox}{max width=.9\textwidth, center}
\includegraphics[width=0.9\textwidth]{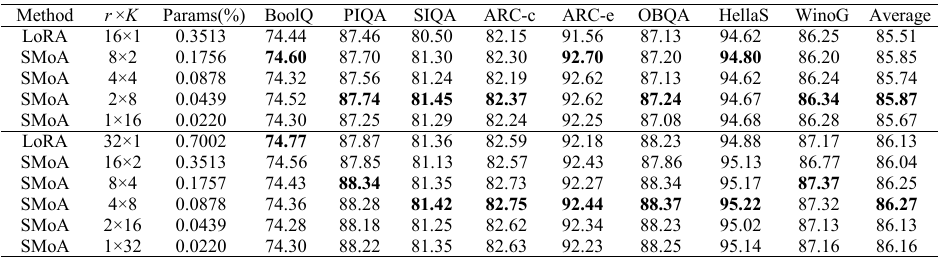}
\end{adjustbox}
\vspace{-2em}
\end{table*}

\subsection{Rank Analysis}

We compare the average rank of the update parameter matrix $\Delta W$ for various PEFT methods,
including HiRA, LoRA, MoRA, and SMoA, all of which employ different strategies for 
adapting to new tasks. As shown in Figure~\ref{fig:rank_pair} (left), SMoA attains significantly higher ranks for
$\Delta W$ across all values of $r$, indicating its superior ability to perform high-rank adaptation.
Compared to LoRA, MoRA, and MeLoRA, SMoA's rank increases dramatically as $r$ increases, 
suggesting that it is more efficient in capturing task-specific information through 
higher-rank updates. Notably, the rank of $\Delta W$ for HiRA is also considerably higher than that for LoRA and MoRA,
though it does not reach the levels of SMoA.
This result highlights HiRA’s effectiveness in improving the expressiveness of the model 
while maintaining a more parameter-efficient approach. Both HiRA and SMoA demonstrate 
that higher-rank updates correlate with improved performance.
In summary, the results emphasize the importance of high-rank adaptation in PEFT methods, 
with SMoA leading in this aspect, achieving substantial improvements over baseline methods, 
and offering a promising direction for enhancing model performance.

\subsection{Effect of Rank on Performance}
Figure~\ref{fig:rank_pair} (right) demonstrates how the performance of SMoA changes with respect 
to the rank parameter $r$ across several commonsense reasoning tasks. As $r$ increases from 2 to 32, we observe a consistent and notable improvement in the model’s performance, with average accuracy rising from 83.49\% to 87.35\%. This trend highlights the impact of increasing the rank $r$ on the model's ability to generalize across diverse reasoning tasks. In particular, tasks such as SIQA, OBQA and WinoG exhibit the most significant performance gains 
as $r$ increases. SIQA, for example, improves by 5.26\% as $r$ rises, emphasizing the importance of a higher rank structure for tasks that demand more sophisticated reasoning capabilities. Notably, even when $r=8$, SMoA demonstrates a competitive edge, suggesting that it efficiently leverages a smaller number of trainable parameters while maintaining strong performance, making SMoA a compelling choice even under constrained resources.
\vspace{-.5cm}

\subsection{Analysis of Block Configurations}
In this section, we analyze how the SMoA block configuration affects performance across different tasks. We evaluate several choices of the pair $(r,K)$, where $r$ is the reference LoRA rank and each block uses local rank $r/K$, and report the results in Table~\ref{tab:rank_analysis}.

We make two observations. First, SMoA remains competitive across a wide range of block configurations. In Table~\ref{tab:rank_analysis}, several $(r,K)$ choices attain strong accuracy while sometimes using more than $2\times$ fewer trainable parameters than stronger-budget alternatives.
This suggests that, in our setting, performance depends not only on the nominal reference rank $r$, but also on how that rank budget is distributed across spectral blocks.
The configuration $(r,K)=(2,8)$ stands out as one favorable trade-off between accuracy and parameter efficiency: it reaches an average score of 85.87\% while keeping the trainable parameter count small. This highlights the importance of choosing a block configuration that balances structural coverage and budget.
Ideally, the pair $(r,K)$ should be tuned for each dataset, although $K=2$ already works well in most cases. Second, the optimal configuration varies across tasks. On ARC-e and HellaSwag, the best performance is usually obtained with $K=2$, whereas other tasks may benefit from larger values such as $K=4$ or $K=8$. We conjecture that this variation is mainly driven by model scale and task complexity, together with how well the reordered block structure matches the downstream residual.


\section{Conclusion}
In this paper, we propose SMoA, a parameter-efficient fine-tuning method that partitions a pretrained operator into multiple aligned spectral blocks and applies one Hadamard-modulated low-rank branch to each diagonal block.
We show that SMoA uses fewer trainable parameters than rank-$r$ LoRA while enlarging the accessible family of spectrum-aware updates through block-diagonal modulation, and we prove both a one-sided structural separation and a spectral approximation gap on a family of block-aligned modulated targets. This connects our method directly to recent LoRA rank theory: rather than only increasing nominal rank, SMoA changes the admissible update family so that some residual directions that would remain in the LoRA spectral tail can be represented exactly. Empirically, SMoA yields stronger average performance together with higher measured update ranks in the lower-budget setting considered here.


\bibliography{bibliography}
\bibliographystyle{plainnat}

\clearpage
\appendix

\section{Datasets}
\label{app:data}
We evaluate three categories of tasks: commonsense reasoning, dialogue generation, and mathematical reasoning. The commonsense reasoning benchmarks include:
\begin{itemize}[leftmargin=*]
    \item \textbf{BoolQ}~\cite{clark2019boolq} is a reading-comprehension dataset of naturally occurring yes/no questions paired with Wikipedia passages, designed to evaluate semantic understanding and reasoning over text.
    \item \textbf{PIQA} (Physical Interaction Question Answering)~\cite{Bisk2020} is a commonsense reasoning benchmark focused on physical knowledge, where the model must choose the more plausible solution to a goal-oriented question involving everyday object interactions and real-world physics.
    \item \textbf{SIQA} (Social Interaction Question Answering)~\cite{sap2019socialiqa} evaluates a model’s understanding of social interactions and human intentions by requiring it to select the most plausible explanation or outcome in everyday situations.
    \item \textbf{ARC-c} (AI2 Reasoning Challenge, Challenge Set)~\cite{clark2018think} is a multiple-choice science benchmark composed of grade-school questions that require nontrivial reasoning and external knowledge.
    \item \textbf{ARC-e} (AI2 Reasoning Challenge, Easy Set)~\cite{clark2018think} is the easier subset of ARC, containing questions that can usually be answered with simpler reasoning or direct factual knowledge.
    \item \textbf{OBQA}~\cite{mihaylov2018can} is a multiple-choice science benchmark in which each question requires combining a small set of ``open-book'' facts with additional commonsense reasoning.
    \item \textbf{HellaSwag}~\cite{zellers2019hellaswag} tests whether a model can predict plausible next events in everyday situations by selecting the most realistic continuation of a given context.
    \item \textbf{WinoG} (WinoGrande)~\cite{ai2:winogrande} is a large-scale commonsense reasoning dataset based on Winograd-style pronoun resolution and is designed to test whether a model can resolve ambiguous references using contextual and commonsense cues.
\end{itemize}
Dataset statistics for the commonsense reasoning benchmarks are summarized in Table~\ref{tab:data_sta}.

For dialogue generation, we use \textbf{ConvAI2}~\cite{dinan2019second}, a persona-based open-domain conversation dataset in which the model must generate or select responses that are coherent, engaging, and consistent with the given persona. For mathematical reasoning, we use \textbf{GSM8K}~\cite{cobbe2021gsm8k}. Following~\citep{huang2025hira}, we train on MetaMath and evaluate on GSM8K~\citep{Cobbe2021TrainingVT}.

\section{Baselines}
\label{app:baselines}
We compare \textsc{SMoA} with a representative set of PEFT methods.
\begin{itemize}[leftmargin=*]
    \item \textbf{Prompt Tuning}~\cite{lester2021power} adapts pretrained language models by learning a small set of trainable prompt embeddings while freezing the backbone parameters.
    \item \textbf{P-Tuning v2}~\cite{liu2021ptuning} injects trainable continuous prompts into every layer of a pretrained language model, yielding stronger parameter-efficient adaptation, especially at larger model scales.
    \item \textbf{LoRA}~\cite{hu2022lora} inserts low-rank trainable matrices into existing linear layers, achieving strong performance while updating only a small number of parameters.
    \item \textbf{DoRA}~\cite{liu2024dora} extends LoRA by decomposing the update into magnitude and direction components, allowing more expressive adaptation with a low parameter count.
    \item \textbf{MoRA}~\cite{jiang2024mora} replaces the standard LoRA structure with a higher-rank parameterization to improve expressive capacity under a constrained budget.
    \item \textbf{SSMLoRA}~\cite{yu2025ssmlora} augments LoRA with state-space-model-based structure to better capture long-range dependencies while keeping the number of trainable parameters small.
    \item \textbf{MeLoRA}~\cite{ren2024melora} uses multiple small low-rank adapters to reduce memory cost while improving adaptation performance and robustness.
    \item \textbf{HiRA}~\cite{huang2025hira} extends LoRA through hierarchical low-rank adaptation, enabling the model to capture both coarse- and fine-grained task-specific information.
\end{itemize}

\section{Evaluation Metric}
\label{app:evaluation_metric}
We evaluate each task family with its standard automatic metrics. For BoolQ, PIQA, SIQA, ARC-c, ARC-e, OBQA, HellaSwag, and WinoGrande, we report test accuracy. For each example, the model generates a free-form answer, which we map to the task label set using task-specific answer keywords (e.g., ``true'' or ``false'' for BoolQ, or the corresponding option strings for multiple-choice tasks). The first valid matched keyword is taken as the prediction; if no valid keyword is found, the prediction is counted as incorrect. This extraction protocol follows prior work and yields a unified evaluation pipeline across all eight commonsense reasoning benchmarks~\cite{hu2023llm,liu2024dora}. The ``Average'' score in Table~\ref{tab:commonsense} is the unweighted mean of the eight task accuracies.

For the conversational benchmark CONVAI2, we report BLEU~\cite{papineni2002bleu}, METEOR, ROUGE-L, and BERTScore~\cite{zhang2019bertscore}. In Table~\ref{tab:results2}, BERT-P, BERT-R, and BERT-F1 denote BERTScore precision, recall, and F1, respectively. When we refer to the average conversational score, we mean the unweighted mean over the reported automatic metrics in Table~\ref{tab:results2}.

For mathematical reasoning on GSM8K, we report final-answer accuracy, i.e., exact-match accuracy after extracting the numeric answer from the generated response using the standard answer format. A prediction is marked correct only when the extracted final answer exactly matches the ground-truth solution. Higher is better for all reported metrics.

\section{Implementation Details}
\label{app:implementation_details}
All experiments are conducted on H800-80GB GPUs.
We follow the training setup of~\cite{huang2025hira} except for learning-rate adjustments.
We implement SMoA on Llama-2-7B and Llama-3-8B using reference LoRA ranks $r=16$ and $r=32$. Unless otherwise stated, we use $K=2$ aligned spectral blocks, so each block uses local rank $\rho=r/K$.
For every adapted layer, the SVD of the frozen pretrained weight is computed once before training to determine the fixed spectrum-aware reordering and the corresponding block partition. This preprocessing step does not introduce additional trainable parameters and is not part of the inference path, although it does add a one-time offline cost before fine-tuning. For a layer $W_0 \in \mathbb{R}^{d_{\text{out}}\times d_{\text{in}}}$ with $q=\min(d_{\text{out}},d_{\text{in}})$, the preprocessing cost is the cost of one truncated/full SVD of that matrix, i.e., on the order of $O(d_{\text{out}}d_{\text{in}}q)$. During training, the total adapter parameter count becomes $K\!\left(\frac{r}{K}\frac{d_{\text{in}}}{K}+\frac{r}{K}\frac{d_{\text{out}}}{K}\right)=\frac{r}{K}(d_{\text{in}}+d_{\text{out}})$, which is a factor of $K$ smaller than a rank-$r$ LoRA layer. At inference time, the learned update can be mapped back to the original coordinates and merged into the base weight, so SMoA does not require a separate multi-branch execution path.
We use the AdamW optimizer~\cite{loshchilov2017decoupled} with a learning rate of 0.001 and a warmup of 1000 steps.
For BoolQ, PIQA, SIQA, ARC-c, ARC-e, OBQA, HellaSwag, and WinoGrande, we fine-tune for 5 epochs and evaluate every 100 steps, selecting the best validation checkpoint. For LoRA, DoRA, MoRA, HiRA, SSMLoRA, MeLoRA, and SMoA, we insert trainable updates into the query, key, value, and projection layers. Because SMoA uses fewer trainable parameters than rank-$r$ LoRA in the current setting, the comparison should be interpreted as a sub-budget comparison rather than a matched-parameter comparison. All main results use the same evaluation pipeline; because they are single-run results, small gaps should be interpreted cautiously. We report the average of 5 random seeds (7,42,123,1234,12345) for SMoA.
The hyperparameter configuration of SMoA is shown in Table~\ref{tab:smoa_hyperparameter} (\appref{app:hyperparameter}). The general hyperparameter configurations are shown in Table~\ref{tab:gene_hyparameters} (\appref{app:hyperparameter}).

\section{Additional Method Details}
\label{app:method-details}

\subsection{Spectrum-Aware Preprocessing Details}
\label{app:method-preprocess}

The spectrum-aware preprocessing step uses the frozen singular structure of $W_0$ only to define the permutations $\pi_{\text{out}}$ and $\pi_{\text{in}}$. These permutations reorder output and input coordinates so that coordinates assigned similar spectral roles become contiguous in the reordered matrix $\widetilde W_0 = P_{\text{out}} W_0 P_{\text{in}}^\top$. In the equal-size setting with $K$ blocks, the reordered row and column intervals $\mathcal{O}_k$ and $\mathcal{I}_k$ are the $k$-th contiguous intervals of lengths $d_{\text{out}}/K$ and $d_{\text{in}}/K$, respectively. The corresponding frozen block anchors are then defined by
\begin{equation}
M_k = \widetilde W_0[\mathcal{O}_k,\mathcal{I}_k]
\in \mathbb{R}^{\frac{d_{\text{out}}}{K}\times \frac{d_{\text{in}}}{K}},
\qquad
k=1,\dots,K.
\end{equation}
These anchors remain fixed during fine-tuning; only the low-rank factors attached to them are trained.

\section{More Results}
\label{sec:more_results}

\begin{figure*}[t]
    \centering

    \includegraphics[width=0.9\linewidth]{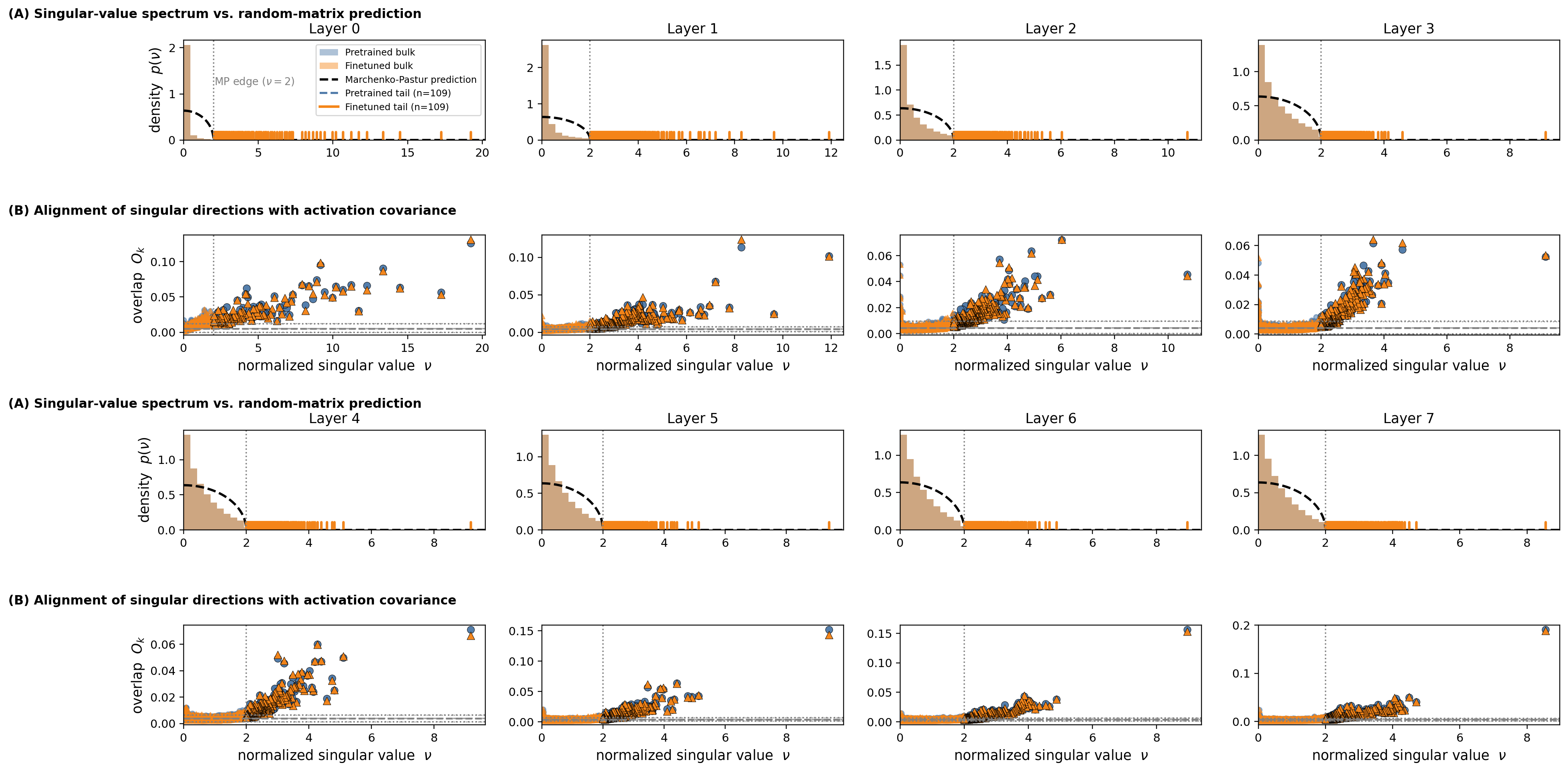}

    \includegraphics[width=0.9\linewidth]{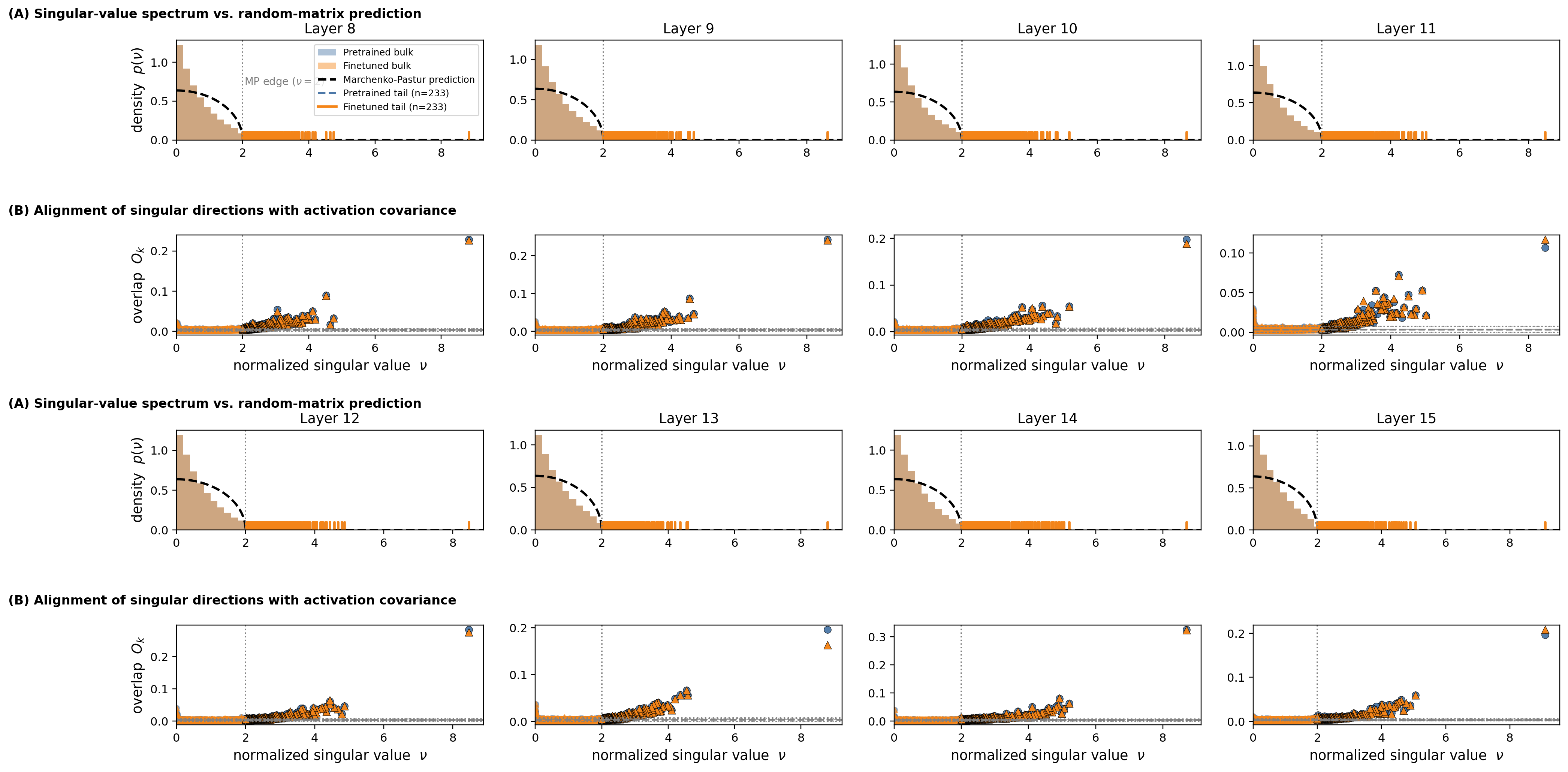}

    \caption{\textbf{Original Llama2-7B vs.\ finetuned spectra and overlaps for key weight matrix} (layer 0 to 15). Top: empirical distributions of normalized singular values $\nu$; light brown denotes the pretrained model and orange denotes the finetuned model. The dashed curve is the theoretical random-matrix bulk. Bottom: the corresponding maximum overlap $O_k$ between right singular vectors and activation-covariance eigenvectors, with circles for the pretrained model and triangles for the finetuned model. Gray lines indicate the bulk mean and $\pm 3\sigma$ intervals. The figure highlights how finetuning changes the right-tail spectrum and the alignment of tail singular directions with task-induced activation structure.}
    \label{fig:spectra_compare_1}
\end{figure*}

\begin{figure*}[t]
    \centering
    \includegraphics[width=0.9\linewidth]{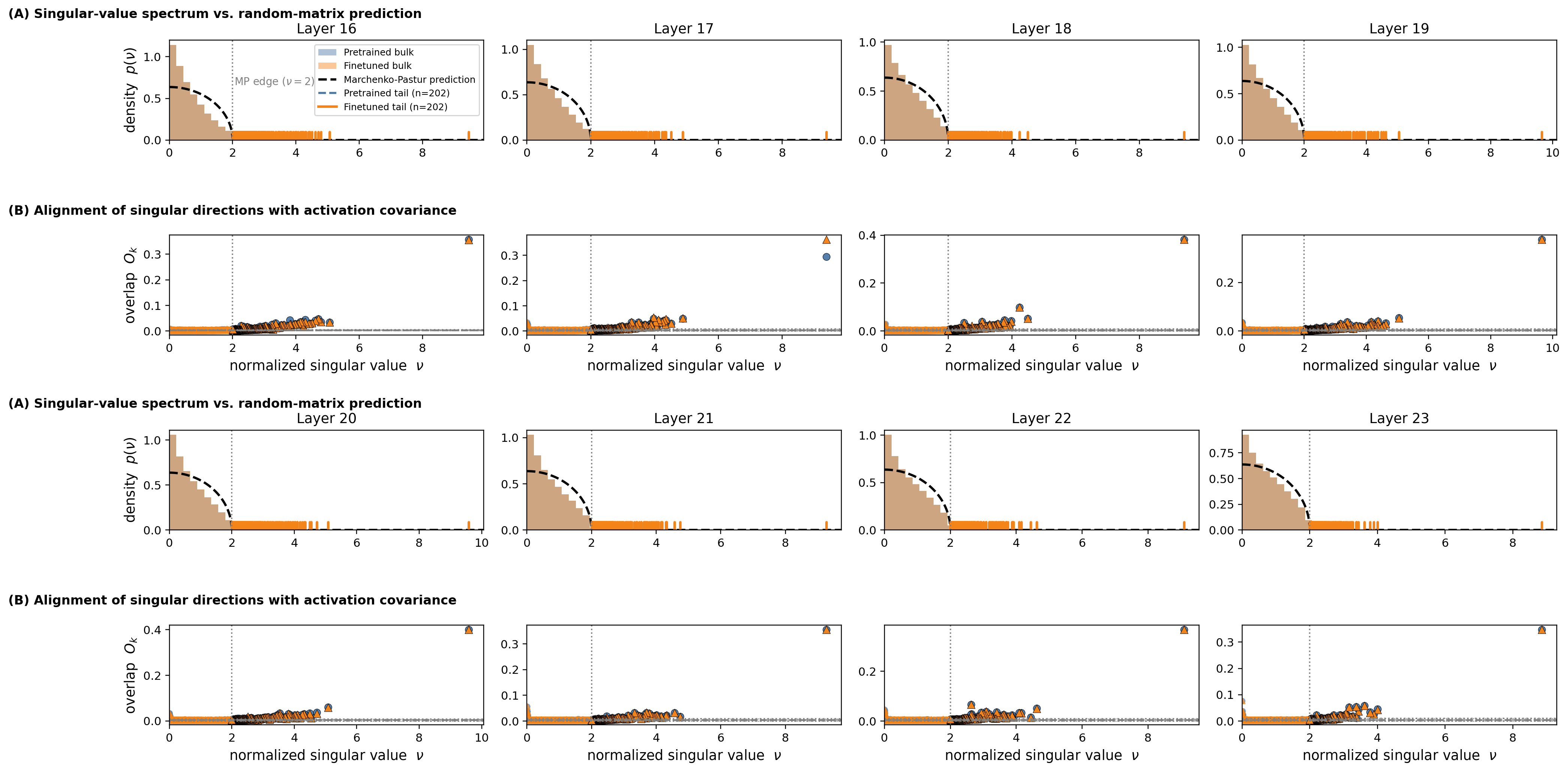}

    \includegraphics[width=0.9\linewidth]{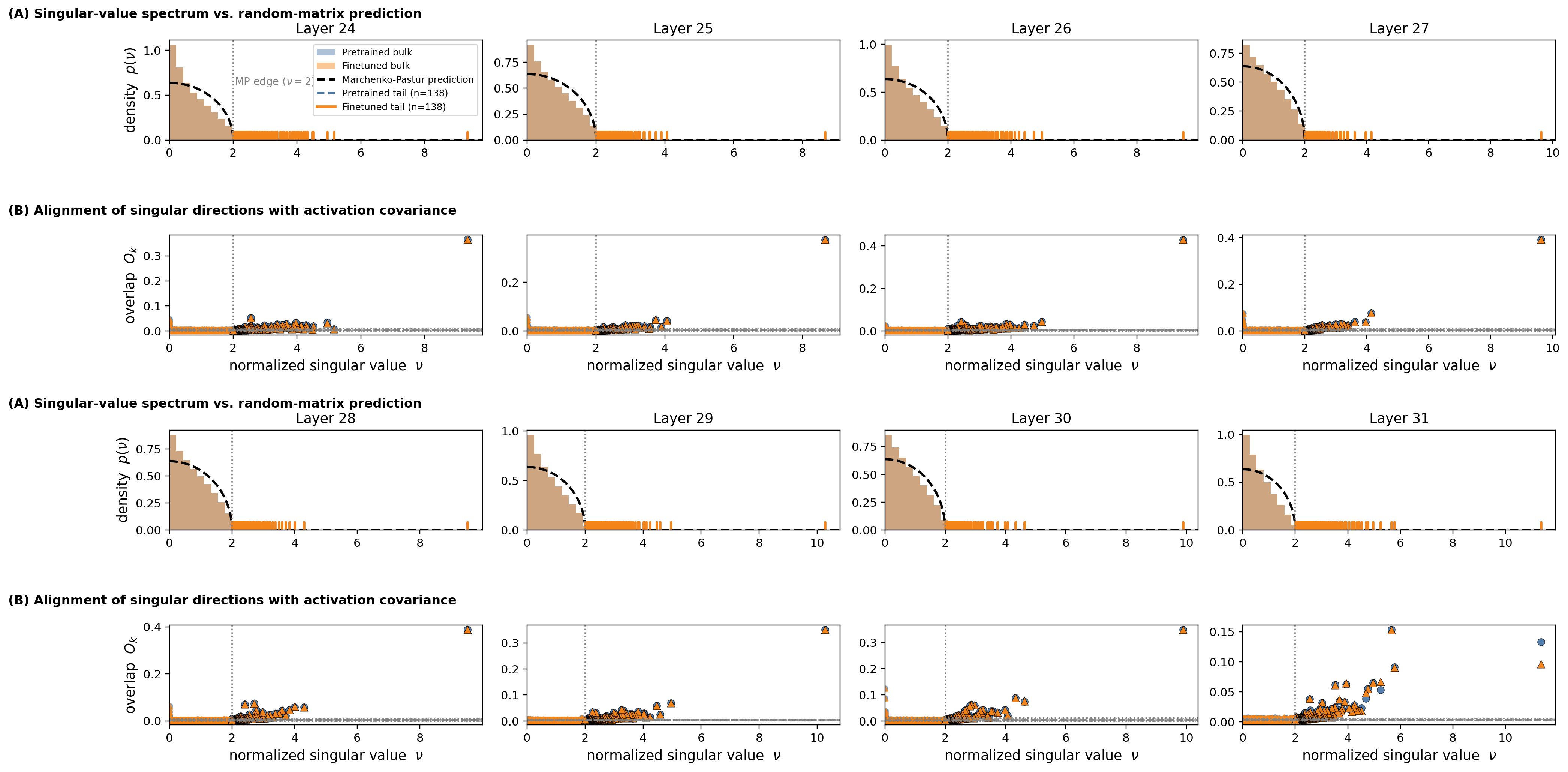}

    \caption{\textbf{Pretrained vs.\ finetuned spectra and overlaps on Llama2-7B key weight matrix} (layer 16 to 31).  Top: empirical distributions of normalized singular values $\nu$; light brown denotes the pretrained model and orange denotes the finetuned model. The dashed curve is the theoretical random-matrix bulk. Bottom: the corresponding maximum overlap $O_k$ between right singular vectors and activation-covariance eigenvectors, with circles for the pretrained model and triangles for the finetuned model. Gray lines indicate the bulk mean and $\pm 3\sigma$ intervals. The figure highlights how finetuning changes the right-tail spectrum and the alignment of tail singular directions with task-induced activation structure.}
    \label{fig:spectra_compare_2}
\end{figure*}

\subsection{Tail Spectrum Visualization}
\label{sub:tail_spectrum_visualization}

\paragraph{Layer-wise spectral changes before and after fine-tuning.}
Figure~\ref{fig:spectra_compare_1} and Figure~\ref{fig:spectra_compare_2} compare the layer-wise spectral distributions of the pretrained model and the finetuned SMoA model. Overall, we observe a consistent and highly structured spectral shift after fine-tuning. In the pretrained model, the singular spectrum is broadly distributed near the low-frequency/small-\(\nu\) region, shown by the light-brown density. After SMoA fine-tuning, the spectrum becomes more concentrated and shifts toward a higher-\(\nu\) band, shown by the orange density. This indicates that SMoA does not uniformly perturb all spectral directions; instead, it selectively amplifies a subset of singular directions that are weak or less dominant in the pretrained model.

At the earliest layers, i.e., Layer 0--3, the spectral change is relatively irregular compared with deeper layers. Layer 0 exhibits only a mild orange response, suggesting that the bottom layer is less affected by task-specific adaptation. From Layer 1 to Layer 3, the finetuned spectrum begins to form a visible bump around intermediate \(\nu\), while the pretrained spectrum still concentrates near the origin. The corresponding \(O_k\) values increase mainly in the intermediate-to-large \(\nu\) region, implying that the learned updates start to align with non-dominant spectral components rather than the principal pretrained directions.

For the shallow-to-middle layers, Layer 4--11, the spectral shift becomes much clearer. The orange density consistently forms a sharp peak around \(\nu \approx 2\), followed by a long but sparse tail. In contrast, the pretrained spectrum remains concentrated in the low-\(\nu\) bulk. This separation suggests that SMoA introduces task-specific energy into directions that are not strongly expressed in the pretrained weights. The overlap plots further support this observation: \(O_k\) is close to zero for most low-\(\nu\) components but grows noticeably in the intermediate and tail regions. In particular, Layers 6--11 show stronger isolated high-\(\nu\) responses, indicating that these layers contain directions highly sensitive to fine-tuning.

In the middle layers, Layer 12--19, the phenomenon becomes even more pronounced. The orange spectrum is highly concentrated around a narrow band, whereas the pretrained density decays smoothly from the origin. This implies that fine-tuning reorganizes the spectral mass into a more compact set of directions. The \(O_k\) curves show that the dominant changes are not located in the pretrained spectral bulk, but rather in the right-side tail. Several layers, such as Layers 14, 17, 18, and 19, exhibit large \(O_k\) outliers at high \(\nu\), suggesting that SMoA learns a small number of highly influential spectral directions. These directions are likely associated with task-discriminative features that are underrepresented in the pretrained model.

For the deeper layers, Layer 20--27, the orange density becomes sharper and more peaked, while the pretrained spectrum keeps a relatively smooth low-\(\nu\) distribution. This indicates that the effect of SMoA fine-tuning is increasingly selective in deeper transformer blocks. The overlap scores in these layers are mostly small in the bulk region but display clear spikes in the spectral tail. Such behavior is consistent with the intuition that later layers encode more task-specific abstractions: fine-tuning does not need to rewrite the entire representation space, but instead adjusts a sparse set of high-impact directions.

Finally, in the top layers, Layer 28--31, we observe the strongest spectral localization. The finetuned spectrum is sharply centered around a narrow \(\nu\) interval, and the pretrained spectrum remains separated from it. The corresponding \(O_k\) values show clear activation in the intermediate-to-large \(\nu\) region, especially in Layer 31, where several pronounced overlap responses appear. This suggests that the last layers are most directly responsible for adapting the pretrained representation to the downstream task. Importantly, the fine-tuning signal is not uniformly distributed across all singular components; rather, it is concentrated in a small number of tail or near-tail directions.

Taken together, these layer-wise observations provide empirical evidence that SMoA fine-tuning mainly modifies the tail and intermediate spectral components of the pretrained model, while leaving the dominant low-\(\nu\) bulk largely intact. This supports the view that parameter-efficient fine-tuning can achieve adaptation by exploiting underutilized spectral directions instead of overwriting the principal pretrained subspace. The consistent emergence of orange peaks across layers further suggests that SMoA learns a structured spectral reallocation, concentrating task-specific information into a compact set of singular directions.
\subsection{Spectral Structure}
\label{sub:spectral_structure}

Figure~\ref{fig:outlier_summary} shows the singular-value histograms of the attention key weight matrix of Llama-2-7B based on SMoA fine-tuning in BoolQ (Similar results can be observed in the Query and Value parameter matrices). Rather than exhibiting a sharply low-rank spectrum, the weights display a clear separation between a bulk component and a set of outlying singular values. The bulk can be interpreted as a high-dimensional noise-like background, while the outliers correspond to structured directions learned during pretraining. Importantly, these outliers are not limited to only the top few singular modes; many layers contain a long tail of moderately large singular values that remain separated from the bulk. This suggests that the key projection matrices encode information in a distributed high-rank subspace, rather than concentrating all useful structure in a small number of dominant directions. Such a spectral pattern provides evidence that fine-tuning may require modifying not only the leading singular components, but also weaker tail directions that capture more fine-grained or task-specific features.
\begin{figure*}[t]
    \centering
    \includegraphics[width=\textwidth]{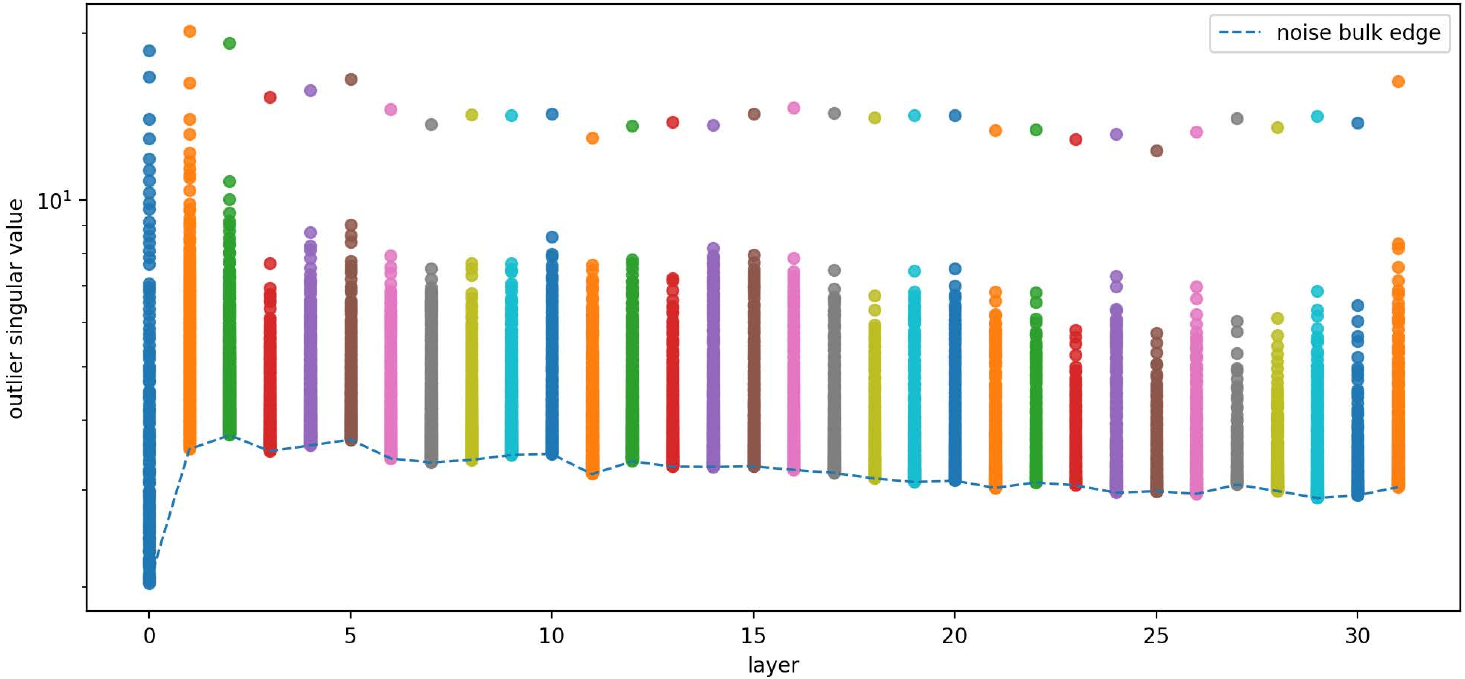}
    \caption{Attention key singular-value histograms of the Llama-2-7B based on SMoA fine-tuning in BoolQ across layers. The spectra exhibit a clear separation between the noise-like bulk and multiple outlier singular values. The persistent outliers, especially those in the spectral tail, indicate that the key projection weights contain rich high-rank structure.}
    \label{fig:outlier_summary}
\end{figure*}

\section{Additional Proofs}
\label{app:proofs}

\begin{lemma}[{Spectral Informativeness of Weights}]
\label{lem:empirical-near-full-rank}

Let $W_0\in\mathbb{R}^{d_{\mathrm{out}}\times d_{\mathrm{in}}}$ be a pretrained Transformer weight matrix with singular values
\begin{equation}
\sigma_1(W_0)\ge \sigma_2(W_0)\ge \cdots \ge \sigma_m(W_0)\ge 0,
\qquad m=\min(d_{\mathrm{out}},d_{\mathrm{in}}).
\end{equation}
Then prior empirical evidence suggests that $W$ exhibits a long-tailed spectrum and is near full-rank in a numerical sense~\cite{staats2024small}, in that
\begin{equation}
\operatorname{rank}_{\varepsilon}(W_0)
=
\#\{i:\sigma_i(W_0)>\varepsilon\}
\approx m,
\end{equation}
while for any $r<m$,
\begin{equation}
\sum_{i=r+1}^{m}\sigma_i^2(W_0)
\not\approx 0.
\end{equation}
\end{lemma}

\subsection{Proof of Lemma~\ref{lem:empirical-near-full-rank}}
\label{app:proof-lemma1}
\begin{proof}
Write the singular value decomposition of $W$ as
\begin{equation}
W_0=U\Sigma V^\top,
\qquad
\Sigma=\operatorname{diag}(\sigma_1,\dots,\sigma_m).
\end{equation}

By the empirical observations of \citet{staats2024small}, the singular-value spectra of pretrained Transformer weight matrices exhibit a long-tailed profile rather than a sharp low-rank collapse. Equivalently, there is no rapid transition of the form
\begin{equation}
\sigma_1\ge \cdots \ge \sigma_r \gg \sigma_{r+1}\approx\cdots\approx \sigma_m\approx 0
\end{equation}
for some small $r\ll m$; instead, a substantial portion of the spectrum remains above numerical noise level. Hence
\begin{equation}
\operatorname{rank}_{\varepsilon}(W_0)
=
\#\{i:\sigma_i(W_0)>\varepsilon\}
\approx m,
\end{equation}
which gives the claimed near-full-rank statement in the numerical sense.

Next, for any $r<m$, the Eckart--Young--Mirsky theorem gives
\begin{equation}
\|W_0-W_0^{(r)}\|_F^2
=
\sum_{i=r+1}^{m}\sigma_i^2(W_0)
=
E_{\mathrm{tail}}(r).
\end{equation}
The same empirical study further shows that removing small singular values can increase perplexity, especially after alignment or downstream fine-tuning. Therefore, the lower-tail singular directions cannot be identified with pure noise in general, which implies that $E_{\mathrm{tail}}(r)$ should not be assumed negligible a priori.

Combining the long-tailed spectrum,
\begin{equation}
\operatorname{rank}_{\varepsilon}(W_0)\approx m,
\end{equation}
with the non-negligible effect of truncating
\begin{equation}
\sum_{i=r+1}^{m}\sigma_i^2(W_0),
\end{equation}
we conclude that pretrained Transformer weight matrices are more faithfully described as long-tailed and numerically near full-rank.

\end{proof}

\subsection{Proof of Proposition~\ref{prop:smoa-rank-ceiling}}
\label{app:proof-rank-ceiling}

\begin{proof}
For each block, the local update is
\begin{equation}
\Delta M_k=(B_kA_k)\odot M_k.
\end{equation}
By the standard rank inequality for Hadamard products,
\begin{equation}
\operatorname{rank}(\Delta M_k)
=
\operatorname{rank}\!\bigl((B_kA_k)\odot M_k\bigr)
\le
\operatorname{rank}(B_kA_k)\,\operatorname{rank}(M_k).
\end{equation}
Since
\begin{equation}
B_kA_k
\in
\mathbb{R}^{\frac{d_{\text{out}}}{K}\times \frac{d_{\text{in}}}{K}}
\end{equation}
is formed from factors of inner dimension $\rho$, we have
\begin{equation}
\operatorname{rank}(B_kA_k)\le \rho.
\end{equation}
Therefore,
\begin{equation}
\operatorname{rank}(\Delta M_k)
\le
\rho\,\operatorname{rank}(M_k).
\end{equation}
On the other hand, since
\begin{equation}
\Delta M_k\in\mathbb{R}^{\frac{d_{\text{out}}}{K}\times\frac{d_{\text{in}}}{K}},
\end{equation}
its rank is also bounded by the block size:
\begin{equation}
\operatorname{rank}(\Delta M_k)\le s_k,
\qquad
s_k=\min\!\left(\frac{d_{\text{out}}}{K},\frac{d_{\text{in}}}{K}\right).
\end{equation}
Combining the two bounds yields
\begin{equation}
\operatorname{rank}(\Delta M_k)
\le
\min\!\left(
s_k,\;
\rho\,\operatorname{rank}(M_k)
\right).
\end{equation}

Now, by construction,
\begin{equation}
\Delta W=P_{\text{out}}^\top\,\widetilde{\Delta W}\,P_{\text{in}},
\qquad
\widetilde{\Delta W}=\operatorname{blkdiag}(\Delta M_1,\dots,\Delta M_K).
\end{equation}
Since permutation matrices preserve rank and rank is additive over block-diagonal matrices,
\begin{equation}
\operatorname{rank}(\Delta W)
=
\operatorname{rank}(\widetilde{\Delta W})
=
\sum_{k=1}^{K}\operatorname{rank}(\Delta M_k)
\le
\sum_{k=1}^{K}
\min\!\left(
s_k,\;
\rho\,\operatorname{rank}(M_k)
\right).
\end{equation}
Define
\begin{equation}
U
:=
\sum_{k=1}^{K}
\min\!\left(
s_k,\;
\rho\,\operatorname{rank}(M_k)
\right).
\end{equation}
Then clearly
\begin{equation}
U\ge 0.
\end{equation}
Moreover, using $\min(a,b)\le a$ for each block,
\begin{equation}
U
\le
\sum_{k=1}^{K}s_k.
\end{equation}
Under the equal-size partition,
\begin{equation}
\sum_{k=1}^{K}s_k
=
K\,\min\!\left(\frac{d_{\text{out}}}{K},\frac{d_{\text{in}}}{K}\right)
=
\min(d_{\text{out}},d_{\text{in}}).
\end{equation}
Hence
\begin{equation}
\operatorname{rank}(\Delta W)
\le
U
\le
\min(d_{\text{out}},d_{\text{in}}).
\end{equation}

For standard LoRA, the update takes the form
\begin{equation}
\Delta W_{\mathrm{LoRA}} = BA,
\qquad
B\in\mathbb{R}^{d_{\text{out}}\times r},
\quad
A\in\mathbb{R}^{r\times d_{\text{in}}},
\end{equation}
and therefore
\begin{equation}
\operatorname{rank}(\Delta W_{\mathrm{LoRA}})
=
\operatorname{rank}(BA)
\le r.
\end{equation}
Thus, whenever
\begin{equation}
U>r,
\end{equation}
the analytic rank ceiling of SMoA is strictly larger than that of rank-$r$ LoRA.
\end{proof}

\subsection{Proof of Corollary~\ref{cor:near-full-rank-ceiling}}
\label{app:proof-full-rank-cor}

\begin{corollary}[{Rank-Ceiling Separation under Full-Rank Local Anchors}]
\label{cor:near-full-rank-ceiling}
Assume the equal-size partition setting and define
\begin{equation}
s=\min\!\left(\frac{d_{\text{out}}}{K},\frac{d_{\text{in}}}{K}\right).
\end{equation}
If each reordered anchor block is full-rank within its local dimensions, i.e.,
\begin{equation}
\operatorname{rank}(M_k)=s,
\qquad
k=1,\dots,K,
\end{equation}
then the SMoA analytic rank upper bound in Equation~\eqref{eq:merged-rank-ceiling} reduces to
\begin{equation}
U_{\mathrm{nf}}
:=
\sum_{k=1}^{K}\min\!\left(s,\frac{r}{K}\operatorname{rank}(M_k)\right)
=
K\min\!\left(s,\frac{r}{K}s\right).
\end{equation}
Moreover, if
\begin{equation}
K<\min(d_{\text{out}},d_{\text{in}})
\qquad\text{and}\qquad
r<\min(d_{\text{out}},d_{\text{in}}),
\end{equation}
then
\begin{equation}
U_{\mathrm{nf}}>r.
\end{equation}
Therefore, under full-rank local anchors, SMoA admits a strictly larger analytic rank ceiling than rank-$r$ LoRA.
\end{corollary}

\begin{proof}
Under the assumption
\begin{equation}
\operatorname{rank}(M_k)=s,
\qquad
k=1,\dots,K,
\end{equation}
Equation~\eqref{eq:merged-rank-ceiling} becomes
\begin{equation}
U_{\mathrm{nf}}
=
\sum_{k=1}^{K}\min\!\left(s,\frac{r}{K}s\right)
=
K\min\!\left(s,\frac{r}{K}s\right).
\end{equation}

We consider two cases.

\medskip
\noindent
\textbf{Case 1:} $r<K$. Then $\frac{r}{K}<1$, so
\begin{equation}
U_{\mathrm{nf}}
=
K\cdot \frac{r}{K}s
=
rs.
\end{equation}
Since
\begin{equation}
K<\min(d_{\text{out}},d_{\text{in}}),
\end{equation}
we have
\begin{equation}
s=\min\!\left(\frac{d_{\text{out}}}{K},\frac{d_{\text{in}}}{K}\right)>1,
\end{equation}
which implies
\begin{equation}
U_{\mathrm{nf}}=rs>r.
\end{equation}

\medskip
\noindent
\textbf{Case 2:} $r\ge K$. Then $\frac{r}{K}\ge 1$, hence
\begin{equation}
U_{\mathrm{nf}}
=
Ks
=
\min(d_{\text{out}},d_{\text{in}}).
\end{equation}
Using the assumption
\begin{equation}
r<\min(d_{\text{out}},d_{\text{in}}),
\end{equation}
we obtain
\begin{equation}
U_{\mathrm{nf}}>r.
\end{equation}

Therefore, in both cases,
\begin{equation}
U_{\mathrm{nf}}>r.
\end{equation}
Since the LoRA analytic rank ceiling is $r$, SMoA admits a strictly larger analytic rank upper bound under the stated full-rank block condition.
\end{proof}

\subsection{Proof of Proposition~\ref{prop:expressivity-separation}}
\label{app:proof-expressivity-separation}

\begin{proposition}[{Expressive Family Separation}]
\label{prop:expressivity-separation}

For a reference rank budget $r$, define the standard LoRA family
\begin{equation}
\mathcal{F}_{\mathrm{LoRA}}(r)
=
\left\{
BA:
A\in\mathbb{R}^{r\times d_{\text{in}}},
\;
B\in\mathbb{R}^{d_{\text{out}}\times r}
\right\}.
\end{equation}
Define also the SMoA family induced by the fixed preprocessing
\(
(P_{\text{out}},P_{\text{in}},\{M_k\}_{k=1}^{K})
\)
as
\begin{equation}
\begin{aligned}
\mathcal{F}_{\mathrm{SMoA}}(W_0;r,K)
=
\Bigl\{
&P_{\text{out}}^\top
\operatorname{blkdiag}\bigl((B_1A_1)\odot M_1,\dots,(B_KA_K)\odot M_K\bigr)
P_{\text{in}}
:\\
& A_k\in\mathbb{R}^{\rho\times d_{\text{in}}/K},\;
B_k\in\mathbb{R}^{d_{\text{out}}/K\times \rho},
\;
k=1,\dots,K
\Bigr\},
\end{aligned}
\end{equation}
where $\rho=r/K$. Now let
\begin{equation}
C_k\in\mathbb{R}^{\frac{d_{\text{out}}}{K}\times\frac{d_{\text{in}}}{K}},
\qquad
\operatorname{rank}(C_k)\le \rho,
\qquad
k=1,\dots,K,
\end{equation}
and define the block-aligned target update in reordered coordinates by
\begin{equation}
\label{eq:block-aligned-target}
\widetilde{\Delta W}^{\star}
=
\operatorname{blkdiag}(C_1\odot M_1,\dots,C_K\odot M_K).
\end{equation}
Then:

\begin{enumerate}
\item
\begin{equation}
P_{\text{out}}^\top \widetilde{\Delta W}^{\star} P_{\text{in}}
\in
\mathcal{F}_{\mathrm{SMoA}}(W_0;r,K).
\end{equation}

\item
If
\begin{equation}
\operatorname{rank}(\widetilde{\Delta W}^{\star})>r,
\end{equation}
then
\begin{equation}
P_{\text{out}}^\top \widetilde{\Delta W}^{\star} P_{\text{in}}
\notin
\mathcal{F}_{\mathrm{LoRA}}(r).
\end{equation}
\end{enumerate}

Therefore, any block-aligned anchor-modulated target whose rank exceeds $r$ serves as a witness that
\begin{equation}
\mathcal{F}_{\mathrm{SMoA}}(W_0;r,K)
\setminus
\mathcal{F}_{\mathrm{LoRA}}(r)
\neq \varnothing.
\end{equation}

\end{proposition}

\begin{proof}

Since
\begin{equation}
\operatorname{rank}(C_k)\le \rho,
\end{equation}
each matrix $C_k$ admits a rank-$\rho$ factorization
\begin{equation}
C_k=B_kA_k,
\qquad
A_k\in\mathbb{R}^{\rho\times d_{\text{in}}/K},
\quad
B_k\in\mathbb{R}^{d_{\text{out}}/K\times \rho}.
\end{equation}
Substituting these factorizations into the SMoA construction gives
\begin{align}
&P_{\text{out}}^\top
\operatorname{blkdiag}\bigl((B_1A_1)\odot M_1,\dots,(B_KA_K)\odot M_K\bigr)
P_{\text{in}} \nonumber\\
&\qquad=
P_{\text{out}}^\top
\operatorname{blkdiag}(C_1\odot M_1,\dots,C_K\odot M_K)
P_{\text{in}}
=
P_{\text{out}}^\top \widetilde{\Delta W}^{\star} P_{\text{in}}.
\end{align}
Hence
\begin{equation}
P_{\text{out}}^\top \widetilde{\Delta W}^{\star} P_{\text{in}}
\in
\mathcal{F}_{\mathrm{SMoA}}(W_0;r,K).
\end{equation}

Next, every matrix in $\mathcal{F}_{\mathrm{LoRA}}(r)$ has rank at most $r$, because for any
\begin{equation}
\Delta W_{\mathrm{LoRA}}=BA
\quad\text{with}\quad
A\in\mathbb{R}^{r\times d_{\text{in}}},
\;
B\in\mathbb{R}^{d_{\text{out}}\times r},
\end{equation}
we have
\begin{equation}
\operatorname{rank}(\Delta W_{\mathrm{LoRA}})
=
\operatorname{rank}(BA)
\le r.
\end{equation}
Moreover, permutation matrices preserve rank, so
\begin{equation}
\operatorname{rank}\!\left(P_{\text{out}}^\top \widetilde{\Delta W}^{\star} P_{\text{in}}\right)
=
\operatorname{rank}(\widetilde{\Delta W}^{\star}).
\end{equation}
Therefore, if
\begin{equation}
\operatorname{rank}(\widetilde{\Delta W}^{\star})>r,
\end{equation}
then
\begin{equation}
\operatorname{rank}\!\left(P_{\text{out}}^\top \widetilde{\Delta W}^{\star} P_{\text{in}}\right)>r,
\end{equation}
which rules out membership in $\mathcal{F}_{\mathrm{LoRA}}(r)$. Thus
\begin{equation}
P_{\text{out}}^\top \widetilde{\Delta W}^{\star} P_{\text{in}}
\notin
\mathcal{F}_{\mathrm{LoRA}}(r).
\end{equation}
This proves the claim.

\end{proof}

\subsection{Proof of Corollary~\ref{cor:spectral-gap}}
\label{app:proof-spectral-gap}

\begin{corollary}[{Spectral Approximation Gap}]
\label{cor:spectral-gap}

Let
\begin{equation}
\widetilde{\Delta W}^{\star}
=
\operatorname{blkdiag}(C_1\odot M_1,\dots,C_K\odot M_K)
\end{equation}
be any block-aligned target from Proposition~\ref{prop:expressivity-separation}, and let
\begin{equation}
\sigma_1(\widetilde{\Delta W}^{\star})
\ge
\sigma_2(\widetilde{\Delta W}^{\star})
\ge \cdots
\end{equation}
denote its singular values. If
\begin{equation}
\operatorname{rank}(\widetilde{\Delta W}^{\star})>r,
\end{equation}
then
\begin{equation}
\inf_{\Delta\in\mathcal{F}_{\mathrm{LoRA}}(r)}
\left\|
\Delta - P_{\text{out}}^\top\widetilde{\Delta W}^{\star}P_{\text{in}}
\right\|_F^2
=
\sum_{j>r}\sigma_j(\widetilde{\Delta W}^{\star})^2
>
0,
\end{equation}
whereas
\begin{equation}
\inf_{\Delta\in\mathcal{F}_{\mathrm{SMoA}}(W_0;r,K)}
\left\|
\Delta - P_{\text{out}}^\top\widetilde{\Delta W}^{\star}P_{\text{in}}
\right\|_F^2
=
0.
\end{equation}
Therefore, rank-$r$ LoRA incurs an irreducible spectral-tail approximation error on this witness family, while SMoA realizes the target exactly.

\end{corollary}

\begin{proof}

By Proposition~\ref{prop:expressivity-separation},
\begin{equation}
P_{\text{out}}^\top\widetilde{\Delta W}^{\star}P_{\text{in}}
\in
\mathcal{F}_{\mathrm{SMoA}}(W_0;r,K),
\end{equation}
and therefore
\begin{equation}
\inf_{\Delta\in\mathcal{F}_{\mathrm{SMoA}}(W_0;r,K)}
\left\|
\Delta - P_{\text{out}}^\top\widetilde{\Delta W}^{\star}P_{\text{in}}
\right\|_F^2
=
0.
\end{equation}

For LoRA, note that
\begin{equation}
\mathcal{F}_{\mathrm{LoRA}}(r)
=
\{X\in\mathbb{R}^{d_{\text{out}}\times d_{\text{in}}}:\operatorname{rank}(X)\le r\}.
\end{equation}
Moreover, permutation matrices preserve both the Frobenius norm and the singular values. Hence
\begin{align}
&\inf_{\Delta\in\mathcal{F}_{\mathrm{LoRA}}(r)}
\left\|
\Delta - P_{\text{out}}^\top\widetilde{\Delta W}^{\star}P_{\text{in}}
\right\|_F^2 \nonumber\\
&\qquad=
\inf_{\operatorname{rank}(X)\le r}
\left\|
X-\widetilde{\Delta W}^{\star}
\right\|_F^2.
\end{align}
By the Eckart--Young--Mirsky theorem, the right-hand side is exactly
\begin{equation}
\sum_{j>r}\sigma_j(\widetilde{\Delta W}^{\star})^2.
\end{equation}
Since
\begin{equation}
\operatorname{rank}(\widetilde{\Delta W}^{\star})>r,
\end{equation}
at least one singular value beyond the first $r$ is nonzero, and thus
\begin{equation}
\sum_{j>r}\sigma_j(\widetilde{\Delta W}^{\star})^2>0.
\end{equation}
This proves the result.

\end{proof}

\section{Training Hyperparameters}
\label{app:hyperparameter}

\begin{table*}[ht]
\caption{Hyperparameters for SMoA.}
\label{tab:smoa_hyperparameter}
\begin{adjustbox}{max width=0.85\textwidth, center}
\includegraphics[width=0.85\textwidth]{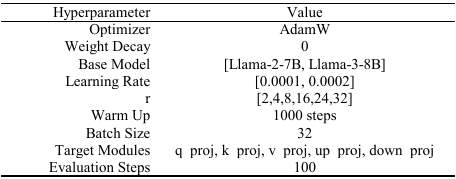}
\end{adjustbox}
\end{table*}

\begin{table*}[ht]
\caption{General hyperparameters used across all experiments.}
\label{tab:gene_hyparameters}
\begin{adjustbox}{max width=0.85\textwidth, center}
\includegraphics[width=0.85\textwidth]{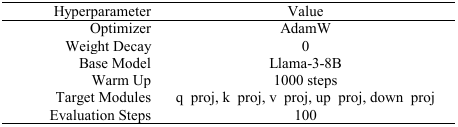}
\end{adjustbox}
\end{table*}

\begin{table*}[ht]
\caption{The detailed statistics of commonsense reasoning datasets.}
\label{tab:data_sta}
\begin{adjustbox}{max width=0.55\textwidth, center}
\includegraphics[width=0.55\textwidth]{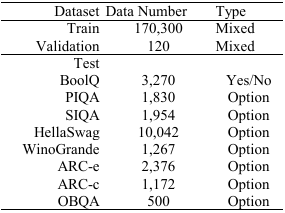}
\end{adjustbox}
\end{table*}

\section{Limitations}
\label{app:limitations}

SMoA requires a one-time spectral preprocessing step for each adapted pretrained weight matrix. This preprocessing does not add trainable parameters or inference-time latency after merging, but it introduces additional offline computation and storage for the fixed permutations and block anchors. The cost may become non-negligible for very large models or for settings where many checkpoints must be adapted independently.
\clearpage
\newpage
\input{checklist.tex}

\end{document}

%% file: checklist.tex
\section*{NeurIPS Paper Checklist}

\begin{enumerate}

\item {\bf Claims}
    \item[] Question: Do the main claims made in the abstract and introduction accurately reflect the paper's contributions and scope?
    \item[] Answer: \answerYes{} 
    \item[] Justification: Section~\ref{sec:method} introduces SMoA and its spectrum-aware block construction, Section~\ref{sec:32} formalizes the parameter count, rank-capacity, and expressivity claims, and Section~\ref{sec:experiments} reports results on different tasks. 

\item {\bf Limitations}
    \item[] Question: Does the paper discuss the limitations of the work performed by the authors?
    \item[] Answer: \answerYes{} 
    \item[] Justification: The paper discusses limitations in Appendix~\ref{app:limitations}.

\item {\bf Theory assumptions and proofs}
    \item[] Question: For each theoretical result, does the paper provide the full set of assumptions and a complete (and correct) proof?
    \item[] Answer: \answerYes{} 
    \item[] Justification: The main theoretical statements are stated in Section~\ref{sec:32}, and the full assumptions and proofs are provided in Appendix~\ref{app:proofs}.

    \item {\bf Experimental result reproducibility}
    \item[] Question: Does the paper fully disclose all the information needed to reproduce the main experimental results of the paper to the extent that it affects the main claims and/or conclusions of the paper (regardless of whether the code and data are provided or not)?
    \item[] Answer: \answerYes{} 
    \item[] Justification: Section~\ref{sec:method} specifies the SMoA construction, and Appendices~\ref{app:data}, \ref{app:baselines}, \ref{app:evaluation_metric}, \ref{app:implementation_details}, and \ref{app:hyperparameter} describe datasets, baselines, metrics, training setup, seeds, and hyperparameters needed to reproduce the main experimental protocol.

\item {\bf Open access to data and code}
    \item[] Question: Does the paper provide open access to the data and code, with sufficient instructions to faithfully reproduce the main experimental results, as described in supplemental material?
    \item[] Answer: \answerYes{} 
    \item[] Justification: We provide an anonymous code repository containing the implementation, scripts, and usage instructions for reproducing the experimental results.

\item {\bf Experimental setting/details}
    \item[] Question: Does the paper specify all the training and test details (e.g., data splits, hyperparameters, how they were chosen, type of optimizer) necessary to understand the results?
    \item[] Answer: \answerYes{} 
    \item[] Justification: Appendix~\ref{app:implementation_details} specifies the adapted layers, optimizer, learning rate, warmup, epochs, validation-checkpoint selection, seeds for SMoA, and evaluation frequency; Appendix~\ref{app:hyperparameter} provides the detailed hyperparameter tables.

\item {\bf Experiment statistical significance}
    \item[] Question: Does the paper report error bars suitably and correctly defined or other appropriate information about the statistical significance of the experiments?
    \item[] Answer: \answerYes{} 
    \item[] Justification: Table~\ref{tab:commonsense} reports SMoA results as mean accuracy with standard deviation, and Appendix~\ref{app:implementation_details} states that SMoA is averaged over five random seeds. 

\item {\bf Experiments compute resources}
    \item[] Question: For each experiment, does the paper provide sufficient information on the computer resources (type of compute workers, memory, time of execution) needed to reproduce the experiments?
    \item[] Answer: \answerYes{} 
    \item[] Justification: This paper reports the GPU type used for all experiments in the Appendix~\ref{app:implementation_details}.
    
\item {\bf Code of ethics}
    \item[] Question: Does the research conducted in the paper conform, in every respect, with the NeurIPS Code of Ethics \url{https://neurips.cc/public/EthicsGuidelines}?
    \item[] Answer: \answerYes{} 
    \item[] Justification: We adhere to all the ethics guidelines.

\item {\bf Broader impacts}
    \item[] Question: Does the paper discuss both potential positive societal impacts and negative societal impacts of the work performed?
    \item[] Answer: \answerNo{} 
    \item[] Justification: The paper focuses on a PEFT algorithm and does not include a dedicated broader-impact discussion of potential positive and negative societal effects. 
    
\item {\bf Safeguards}
    \item[] Question: Does the paper describe safeguards that have been put in place for responsible release of data or models that have a high risk for misuse (e.g., pre-trained language models, image generators, or scraped datasets)?
    \item[] Answer: \answerNA{} 
    \item[] Justification: This paper poses no such risks.

\item {\bf Licenses for existing assets}
    \item[] Question: Are the creators or original owners of assets (e.g., code, data, models), used in the paper, properly credited and are the license and terms of use explicitly mentioned and properly respected?
    \item[] Answer: \answerYes{} 
    \item[] Justification: Existing datasets and code are properly cited.

\item {\bf New assets}
    \item[] Question: Are new assets introduced in the paper well documented and is the documentation provided alongside the assets?
    \item[] Answer: \answerNA{} 
    \item[] Justification:  Paper does not exist new assets.

\item {\bf Crowdsourcing and research with human subjects}
    \item[] Question: For crowdsourcing experiments and research with human subjects, does the paper include the full text of instructions given to participants and screenshots, if applicable, as well as details about compensation (if any)? 
    \item[] Answer: \answerNA{} 
    \item[] Justification: This paper does not involve human subjects.

\item {\bf Institutional review board (IRB) approvals or equivalent for research with human subjects}
    \item[] Question: Does the paper describe potential risks incurred by study participants, whether such risks were disclosed to the subjects, and whether Institutional Review Board (IRB) approvals (or an equivalent approval/review based on the requirements of your country or institution) were obtained?
    \item[] Answer: \answerNA{} 
    \item[] Justification: This paper does not involve crowdsourcing or human subjects.

\item {\bf Declaration of LLM usage}
    \item[] Question: Does the paper describe the usage of LLMs if it is an important, original, or non-standard component of the core methods in this research? Note that if the LLM is used only for writing, editing, or formatting purposes and does \emph{not} impact the core methodology, scientific rigor, or originality of the research, declaration is not required.
    \item[] Answer: \answerNA{} 
    \item[] Justification: This paper does not involve LLMs in any meaningful capacity.

\end{enumerate}